%% file: iclr2026_conference.tex
\documentclass{article} 
\usepackage[T1]{fontenc}
\usepackage{iclr2026_conference,times}

\input{math_commands.tex}

\usepackage{hyperref}
\usepackage{url}
\usepackage{algorithm}
\usepackage{algorithmic}
\usepackage{amsmath, amsthm}
\usepackage{graphicx}
\usepackage{float}
\usepackage{subfigure}
\usepackage{booktabs}
\usepackage{threeparttable}
\usepackage{tcolorbox}
\usepackage[svgnames]{xcolor}
\usepackage{wrapfig}
\definecolor{darkmagenta}{rgb}{0.56, 0.0, 1.0}  
\usepackage[commandnameprefix=always]{changes}
\usepackage{xcolor} 
\definecolor{myorange}{rgb}{1,0.5,0} 
\newcommand{\myadded}[1]{\textcolor{black}{#1}} 

\usepackage{amssymb}
\usepackage{xspace}

\definecolor{Cornsilk}{rgb}{1.0, 0.97, 0.86}
\definecolor{lightorange}{rgb}{0.996, 0.855, 0.643}
\definecolor{lightgray}{rgb}{1.0, 0.827, 0.278}

\newcommand{\methodName}{\texttt{SWD}\xspace}

\hypersetup{colorlinks,linkcolor={blue},citecolor={darkmagenta},urlcolor={blue}}

\title{The Rank and Gradient Lost in Non-stationarity: Sample Weight Decay for Mitigating Plasticity Loss in Reinforcement Learning}

\author{
Zihao Wu$^{1,2}$,
Hongyao Tang$^{1}$\thanks{Corresponding authors: 
tanghongyao@tju.edu.cn, yanzheng@tju.edu.cn.},
Yi Ma$^{1}$,
Jiashun Liu$^{3}$,
Yan Zheng$^{1}$\footnotemark[1],
Jianye Hao$^{1}$\\
$^{1}$School of Computer Software, Tianjin University\\
$^{2}$International Joint Institute of Tianjin University, Fuzhou\\
$^{3}$The Hong Kong University of Science and Technology
}

\usepackage{amsthm}
\newtheorem{Proposition}{Proposition}
\newtheorem{theorem}{Theorem}
\newtheorem{lemma}{Lemma}

\iclrfinalcopy
\begin{document}

\maketitle
\vspace{-4mm}
\begin{abstract}
Deep reinforcement learning (RL) suffers from plasticity loss severely due to the nature of non-stationarity, which impairs the ability to adapt to new data and learn continually. Unfortunately, our understanding of how plasticity loss arises, dissipates, and can be dissolved remains limited to empirical findings, leaving the theoretical end underexplored.
To address this gap, we study the plasticity loss problem from the theoretical perspective of network optimization.
By formally characterizing the two culprit factors in online RL process: the non-stationarity of data distributions and the non-stationarity of targets induced by bootstrapping, 
our theory attributes the loss of plasticity to two mechanisms: the rank collapse of the Neural Tangent Kernel (NTK) Gram matrix and the $\Theta(\frac{1}{k})$ decay of gradient magnitude. 
The first mechanism echoes prior empirical findings from the theoretical perspective and sheds light on the effects of existing methods, e.g., network reset, neuron recycle, and noise injection.
Against this backdrop, we focus primarily on the second mechanism and aim to alleviate plasticity loss by addressing the gradient attenuation issue, which is orthogonal to existing methods.
We propose Sample Weight Decay (\methodName) --- a lightweight method to restore gradient magnitude, as a general remedy to plasticity loss for deep RL methods based on experience replay.
In experiments, we evaluate the efficacy of \methodName upon TD3, \myadded{Double DQN} and SAC with SimBa architecture in MuJoCo, \myadded{ALE} and DeepMind Control Suite tasks.
The results demonstrate that \methodName effectively alleviates plasticity loss and consistently improves learning performance across various configurations of deep RL algorithms, UTD, network architectures, and environments, achieving SOTA performance on challenging DMC Humanoid tasks.We make our code available\footnote{\url{https://github.com/wzhhasadream/CleanRL-JAX}}

\end{abstract}

\section{Introduction}

\begin{wrapfigure}[18]{R}{0.55\textwidth}
\vspace{-0.3cm}
    \centering
    \subfigure[SWD for SimBa-SAC in DMC tasks]{
\label{Fig:sac intro}
\includegraphics[width=0.55\textwidth]{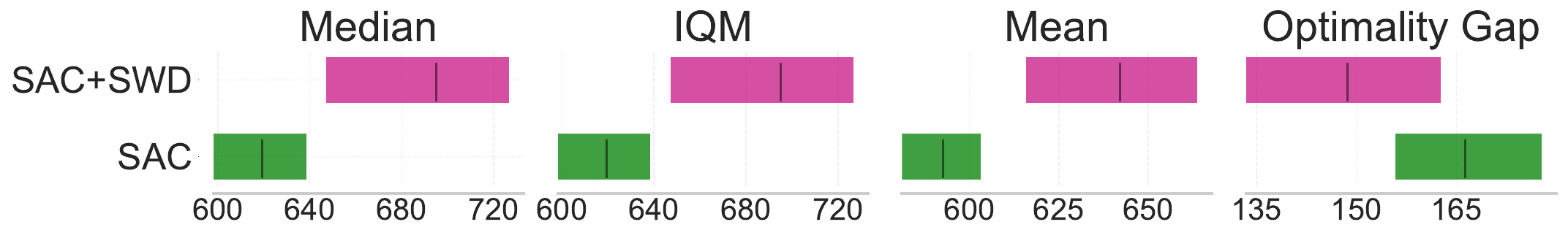}
}
\subfigure[SWD for TD3 in MuJoCo tasks]{
\label{Fig:td3 intro}
\includegraphics[width=0.55\textwidth]{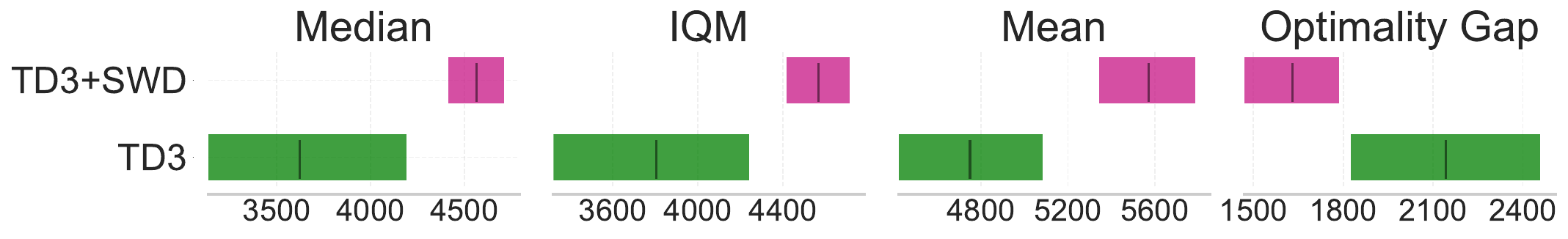}
}
\subfigure[\myadded{SWD for Double DQN in ALE tasks}]{
\label{Fig:ddqn intro}
\includegraphics[width=0.55\textwidth]{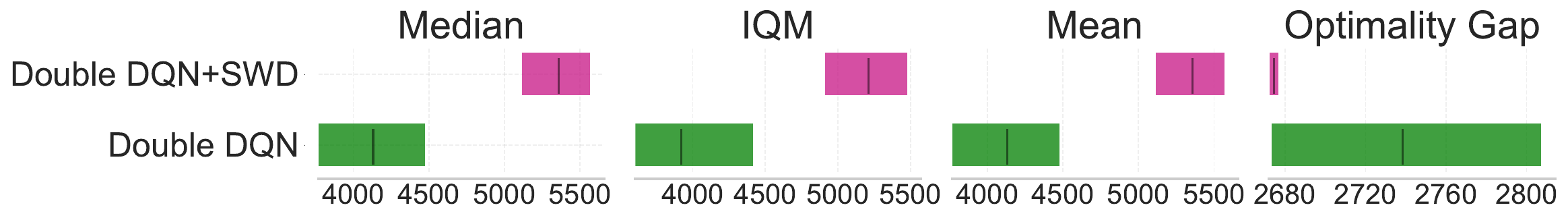}
}
\vspace{-5mm}
    \caption{Aggregate \textit{Reliable} metrics~\citep{agarwal2021deep} with 95\% Stratified Bootstrap CIS.}
    \label{Fig:intro}
\vspace{-1cm}
\end{wrapfigure}
Deep reinforcement learning (RL) has achieved remarkable success across a variety of domains, including robotics~\citep{akkaya2019solving}, game playing~\citep{berner2019dota} and LLM post-training that endows language models with the ability to generate human-like replies for breaking the Turing test~\citep{biever2023chatgpt}. The core driver behind these advancements of deep RL lies in the combination of RL and deep neural networks. With the powerful expressive capacity and adaptive learning ability, the neural networks can effectively approximate and optimize value functions and policies under the RL training regime. However, recent studies have identified a critical yet often overlooked challenge --- \textit{Plasticity Loss}: as training progresses, the learning ability of neural networks gradually diminishes~\citep{elsayed2024addressing,nikishin2022primacy}. 
To address this phenomenon, researchers in the RL community have proposed different metrics and remedies mainly from empirical perspectives, such as Network Reset~\citep{nikishin2022primacy}, Neuron Recycling~\citep{sokar2023dormant}, Noise Injection~\citep{nikishin2023deep}.
However, these existing works all rely on empirical intuitions and lack clear theoretical grounding, leaving a significant gap between empiricism and theory.
Despite the significance of this issue, explaining plasticity from the theoretical perspective and developing principled algorithms remain highly challenging due to the complexity of the underlying mechanisms of plasticity loss in the context of deep RL.

To analyze the optimization dynamics of Reinforcement Learning (RL) agents, we develop a structured theoretical framework rooted in a core insight: due to the dynamic nature of the optimization process in RL, the loss function evolves with each optimization iteration—effectively initiating a new optimization "task" in each round. Critically, the initial optimization point for the updated loss function in the current round is exactly the terminal point from optimizing the previous round’s loss function. This sequential initialization mechanism raises fundamental questions about its potential adverse impacts on optimization performance, and this line of inquiry underpins the entire logic of our theoretical analysis.
Based on this insight, we arrive at a key conclusion: RL agents inherently confront two critical challenges that exert profound adverse effects on loss function optimization.The first is the potential rank deficiency of the Neural Tangent Kernel (NTK) \citep{jacot2018neural}—a core factor that governs the network’s fitting capacity, specifically its ability to approximate the optimal value function in RL.The second is a gradient magnitude decay , which directly regulates the neural network’s fitting rate and dictates the time required to escape saddle points.

Our theoretical results reveal two causal mechanisms for the occurrence of plasticity loss.
The first mechanism echoes prior empirical findings from the theoretical perspective and sheds light on the effects of existing methods.
Differently, we focus primarily on the second mechanism, which has not been well explored, and aim to alleviate plasticity loss by addressing the gradient attenuation issue from an orthogonal angle to existing methods.
In this paper, we design an anti-decay sampling strategy as a compensation measure. We observe that gradient decay is governed by the linearly decaying term $\frac{1}{k}$, where 
$k$ represents the number of learning iteration. In response to this, we construct a set of linearly weighted coefficients, where the sampling probability decreases linearly with the \textit{age} of the samples. 
Specifically, we propose \textbf{Sample Weight Decay (\methodName)} --- a lightweight method tailored to mitigate plasticity loss in deep Reinforcement Learning (RL) algorithms. \methodName effectively maintains the gradient magnitude at a appropriate scale, ensuring stable learning dynamics.

Building on the SimBa-SAC~\citep{lee2025simba,sac} , TD3~\citep{td3}  and \myadded{Double DQN}~\citep{DDQN} algorithms as base
algorithm, \methodName significantly enhances learning stability and performance in continuous control tasks and \myadded{pixel-based tasks}. To validate its effectiveness, we evaluated \methodName across three well-established online reinforcement learning (RL) benchmarks: the MuJoCo~\citep{MuJoCo} , \myadded{Arcade Learning Environment}~\citep{ale} and the DeepMind Control (DMC) Suite~\citep{dmc}. For our evaluation protocol, we adopted the Interquartile Mean (IQM) as the core performance metric, while leveraging \texttt{GraMa}~\citep{grama} as the key indicator to quantify plasticity. As illustrated in Figure~\ref{Fig:intro}, \methodName consistently delivers state-of-the-art (SOTA) performance.


The contributions of this paper are summarized as follows:
\begin{itemize}
\vspace{-2mm}
\item We have developed a unified theory to account for plasticity in deep reinforcement learning (RL), thereby shedding clear light on the origins of such plasticity, 
bridging the gap between empirical practice and theoretical research.
\vspace{-2mm}
\item We propose \methodName, a theoretically grounded plug-and-play method to different RL algorithms for mitigating plasticity loss and improving learning performance. 
\vspace{-3mm}
\item The experiments demonstrate the efficacy of \methodName in improving learning stability and performance. Additionally, \methodName achieves state-of-the-art (SOTA) performance in challenging DMC Humanoid tasks.
\end{itemize}

\section{Related Work}


Plasticity loss refers to the phenomenon in neural network training where the model gradually loses its ability to adapt to new data, objectives, or tasks during the learning process \citep{DohareHLRMS24LossofPlasticity}. This usually reflects that the network becomes overly specialized to the early stages of training, resulting in reduced learning capacity, slower convergence, or even a collapse in later stages of training~\citep{nikishin2022primacy,nikishin2023deep}. 
To gain a better understanding of plasticity loss and address it effectively, many efforts have been made to conduct various empirical investigations and propose different solutions~\citep{s&p,Lewandowski2023DirectionsOC,kumar2023continual,ceron2023small,AsadiFS23ResetOpt,Ellis2024AdamOnLocalTime,chung2024parseval,Tang2024ReduceChurn,frati2024reset,ceronvalue}. 

\cite{sokar2023dormant} first identified the \textit{dormant neuron phenomenon} in deep reinforcement learning (RL) networks, where neurons progressively fall into an inactive state and their expressive capacity diminishes over the course of training. To address this issue, they proposed Recycle Dormant neurons (ReDo) --- a strategy that continuously detects and recycles dormant neurons throughout the training process. In a separate line of work, \cite{nikishin2023deep} proposed Plasticity Injection, a minimal-intervention technique that boosts network plasticity without altering trainable parameters or introducing biases into predictive outputs. More recently, \cite{grama} introduced Reset guided by Gradient Magnitude (ReGraMa), which addresses neuronal activity loss in deep RL agents by transitioning from activation statistics to gradient-based neuron reset strategies, maintaining network plasticity through \texttt{GraMa} metrics.
While these approaches have empirically validated their effectiveness in combating plasticity loss, they predominantly operate at the model level --- modifying network architectures without addressing the fundamental theoretical questions: \textit{why} plasticity loss occurs and \textit{how} different underlying mechanisms contribute to this phenomenon. 
This presents a significant gap
between empiricism and theory.

This theoretical gap motivates our work, which targets the fundamental {gradient decay mechanism} identified through our theoretical analysis. Our proposed Sample Weight Decay (\texttt{SWD}) approach operates at the {strategic level} --- focused on weighting in experience replay --- and provides a principled means of compensating for the $\Theta(1/k)$ gradient attenuation, a challenge unaddressed by recent techniques. A key distinguishing feature of \texttt{SWD} is its \textit{orthogonality} to existing methods: whereas prior approaches modify network structures or plasticity injection patterns, \texttt{SWD} acts at the data distribution level via intelligent experience reweighting, ensuring compatibility with existing plasticity-preserving techniques and enabling synergistic performance improvements.

\section{Preliminaries}


We consider an episodic Markov Decision Process (MDP) $(\sS,\sA,H,\{P_h\}_{h=1}^{H},\{r_h\}_{h=1}^{H})$ with horizon $H\in \sZ^{+}$~\citep{puterman2014markov}.
Here, $\sS,\sA$ are measurable state, action spaces; $P_h(\cdot\!\mid s,a)$ is the transition kernel at step  $h$ ; $r_h:\sS\times\sA\to[0,1]$ is the reward at step $h$.
At each episode, an initial state $x_1$ is drawn. At step $h\in[H]$, the agent observes $x_h\in\sS$, chooses $a_h\in \sA$, receives $r_h(x_h,a_h)$, and transits to $x_{h+1}\sim P_h(\cdot\!\mid x_h,a_h)$.
A  policy is $\pi=\{\pi_h\}_{h=1}^{H}$ with $\pi_h(\cdot\!\mid x)$.

For policy $\pi$, the value and action-value functions are defined as:
\vspace{-1mm}
\begin{align*}
& V_h^\pi(x) = \mathbb E\!\left[\sum_{t=h}^{H} r_t(x_t,a_t) \,\middle|\, x_h=x,\ a_t\sim \pi_t(\cdot\!\mid x_t)\right], 
&& \forall x\in \sS,\ h\in[H] \\
&Q_h^\pi(x,a) = r_h(x,a) + \mathbb E_{x'\sim P_h(\cdot\mid x,a)}\!\big[V_{h+1}^\pi(x')\big],
&& \forall (x,a)\in \sS\times \sA,\ h\in[H], \label{eq:def-Q}
\end{align*}
with terminal condition $V_{H+1}^\pi\equiv 0$.
It is convenient to write the transition expectation operator $\mathbb{P}_h$ and policy expectation operator $\mathbb{J}_{h}^{\pi}$:
\[
(\mathbb{P}_h V)(x,a) = \mathbb E_{x'\sim P_h(\cdot\mid x,a)}[V(x')],
\quad 
(\mathbb{J}_h^\pi Q)(x) = \mathbb E_{a\sim \pi_h(\cdot\mid x)}[Q(x,a)].
\]
Then the policy Bellman equations compactly read,
\begin{align*}
Q_h^\pi(x,a) &= r_h(x,a) + (\mathbb{P}_h V_{h+1}^\pi)(x,a)\\
V_h^\pi(x)   &= (\mathbb{J}_h^\pi Q_h^\pi)(x), \qquad V_{H+1}^\pi \equiv 0. 
\end{align*}

For any function $g:\sS\times\sA\to\mathbb R$, define the value maximization operator $\mathbb V$ and the step-$h$ optimality Bellman operator $\mathcal T_h$ by
\begin{align*}
    \mathbb V_g(x) &:= \max_a g(x,a)\\
    (\mathcal T_h g)(x,a)&:= r_h(x,a)+(\mathbb P_h \mathbb V_g)(x,a) .
\end{align*}

\section{Theory Analysis: The Rank Loss and Gradient Attenuation}\label{sec:loss function}

In this section, our primary objective is to establish a rigorous connection between the optimization process and plasticity loss. To this end, we first utilize Equation~\ref{eq:bound} to derive a formal bound on the model’s performance. We then simplify the dynamic optimization process by reducing it to an initialization problem—a key step that streamlines subsequent analyses. Finally, we elaborate on the derivation of our core results, with the full details presented in Section~\ref{sec:ntk} and Section~\ref{sec:grad}.  

For the sake of clarity and analytical tractability, we focus our discussion on the simplest variant of Fitted Q-Iteration (FQI) \citep{FQI}. Importantly, the theoretical framework proposed herein is not limited to this specific algorithm; it can be readily extended to accommodate a wider class of value-based reinforcement learning methods. Of note, analogous analytical findings hold for entropy-regularized Markov Decision Processes (MDPs). A comprehensive treatment of this extension, including detailed proofs and supplementary analyses, is provided in Appendix~\ref{App:entropy mdp}.

Let $\mathcal{D}_h^k$ denote the replay buffer at step $h$ following $k$ episodes, and let $\hat{f}_{h+1}^k$ represent the estimated Q-value at step $h+1$ after $k$ episodes. The loss function is then defined as follows:
\begin{align*}
\mathcal{L}_{h}^k(f, \hat{f}_{h+1}^k) &:=\frac{1}{|\mathcal{D}_h^k|} \sum_{(s_h, a_h, s_{h+1}) \sim \mathcal{D}_h^k} \left[ \left(f(s_h, a_h) - \left(r(s_h, a_h) + \max_{a'} \hat{f}_{h+1}^k(s_{h+1}, a')\right)\right)^2 \right]  \\
\hat f_{h}^k &= \arg\min_{f\in \mathcal F} \mathcal{L}_{h}^k(f, \hat{f}_{h+1}^k) ,\quad \hat f_{H+1} \equiv 0\label{eq:argmin}
\end{align*}
Define the empirical distribution $\mu_h^k$ of the replay buffer over $(s,a)$ and the empirical state-action visitation frequency of the behavior policy $\pi^{k+1}$ at time $h$ in episode $k$:
\begin{align*}
\mu_h^k(s,a) &:= \frac{1}{|\mathcal{D}_h^k|}\sum_{(s_i,a_i,s'_i)\in \mathcal{D}_h^k} \mathbb{I}\{s = s_i,\ a = a_i\}, \\
\hat{d}_h^{\pi^{k+1}}(s, a) &:= \mathbb I\{s=s_h^{k+1},,a=a_h^{k+1}\},
\qquad (s_h^{k+1},a_h^{k+1})\sim \mathbb{P}_h^{\pi^{k+1}}(s,a)
\end{align*}
%
To establish a mathematical formulation for distribution shift and thereby quantify its impact on the loss function, we rely on Proposition~\ref{pro:rec} to characterize such distributional non-stationarity. Furthermore, to facilitate the subsequent gradient decomposition, we express the loss function in the form specified in Theorem~\ref{thm:pll}. Finally, to connect the agent’s performance to the loss function, we leverage Theorem~\ref{thm:subopt-f} to provide a bound on the agent’s final performance.
\begin{Proposition}[Empirical distribution recursion]\label{pro:rec}
The empirical distribution satisfies
\begin{equation}
\mu_h^{k+1}=\frac{k}{k+1}\,\mu_h^k+\frac{1}{k+1}\,\hat d_h^{\pi^{k+1}}.
\end{equation}
\end{Proposition}
\begin{proof}[Proof (sketch)]
By construction, $|\mathcal D_h^{k+1}|=k+1$ and $\mathcal D_h^{k+1}=\mathcal D_h^k\cup\{(s_h^{k+1},a_h^{k+1},s_{h+1}^{k+1})\}$. Expanding the definition of $\mu_h^{k+1}$ and regrouping terms yields the stated convex combination.
\end{proof}
\begin{theorem}[Population loss limit]\label{thm:pll}
Let $\mathcal{F}$ be a measurable function class. As the cardinality (or appropriate size measure) of $\mathcal{D}_h^k$ tends to infinity (i.e., $|\mathcal{D}_h^k| \to \infty$), the following probabilistic convergence holds:
\begin{equation}\label{eq:expected loss}
\mathcal{L}_{h}^k(f, \hat{f}_{h+1}^k) \xrightarrow{p} \mathbb{E}_{(s_h, a_h) \sim \mu_h^k} \left[ \left( f(s_h, a_h) - (\mathcal{T}_h \hat{f}_{h+1}^k)(s_h, a_h) \right)^2 \right] + C_h^k
\end{equation}
where $C_h^k$ is a constant independent of $f$. Henceforth, we do not rigorously distinguish between the empirical risk and the expected loss, focusing instead on the underlying optimization problem.
\end{theorem}

\begin{tcolorbox}[colback=Cornsilk, colframe=lightgray, title=\textbf{Takeaway 1. The non-stationarity in training process}]
The population loss limit established in Theorem~\ref{thm:pll} identifies two key sources of non-stationarity in the training process of Fitted Q-Iteration: the non-stationary distribution 
$\mu_h^k$ and the non-stationary target 
$ \mathcal{T}_h \hat{f}_{h+1}^k$. Both of these sources drive variations in the target population risk across training episodes k.
\end{tcolorbox}
\begin{theorem}[Suboptimality bound via squared bellman residuals]\label{thm:subopt-f}
Fix horizon $H$. Let $\{\hat f_h\}_{h=1}^H$ denote the final value estimates (e.g., from the $K$-th iteration; write $\hat f_h:=\hat f_h^{(K)}$). Define the greedy policy
$$\pi_{\hat f,h}(s)\in\arg\max_{a\in\sA}\ \hat f_h(s,a),\qquad h=1,\dots,H.$$
For functions $f,g:\sS\times\sA\to\mathbb R$, define the step-$h$ squared Bellman residual
$$\Delta_h(f,g)(s,a)\;:=\;\big(f(s,a)-(\mathcal T_h g)(s,a)\big)^2$$
Then for any start state $x$,
\begin{align}\label{eq:bound}
V_1^*(x) - V_1^{\pi_{\hat{f}}}(x)
&\leq \sqrt{H} \left(
    \sqrt{ \mathbb{E}_{\pi^*} \!\left[ \sum_{h=1}^H \Delta_h(\hat{f}_h, \hat{f}_{h+1})(s_h, a_h) \,\bigg|\, s_1 = x \right] }
    + \right. \nonumber \\
&\qquad \left.
    \sqrt{ \mathbb{E}_{\pi_{\hat{f}}} \!\left[ \sum_{h=1}^H \Delta_h(\hat{f}_h, \hat{f}_{h+1})(s_h, a_h) \,\bigg|\, s_1 = x \right] }
\right).
\end{align}
\end{theorem}


\begin{tcolorbox}[colback=Cornsilk, colframe=lightgray, title=\textbf{Takeaway 2. Suboptimality bound}]
Equation~\ref{eq:bound} links model performance to the loss function for optimization. This means the agent’s performance depends on Bellman residuals from the current and optimal policy trajectories, not historical ones.
\end{tcolorbox}

Building on the foundational framework established above, we can observe that the loss function evolves at each iteration $k$; this phenomenon is analogous to initiating an entirely new round of training. A key distinction from supervised learning lies in the initialization process: whereas supervised learning relies on \textbf{random initialization}, reinforcement learning (RL) commences optimization from the $\arg\min$ of the loss function obtained in the previous iteration. Consequently, investigating the properties of initialization points under such non-steady-state conditions becomes particularly crucial—an insight that further underscores the necessity of the theoretical exploration presented in our work.

\subsection{Neural Tangent Kernel (NTK) Degeneration}\label{sec:ntk}

A key advantage of random initialization is that it ensures the Neural Tangent Kernel (NTK) matrix of an overparameterized neural network is \textbf{full-rank with probability 1}. This comes from the property that low-dimensional manifolds have zero measure in high-dimensional spaces, making the NTK matrix rank-deficient extremely unlikely. However, this random initialization is violated in Reinforcement Learning (RL), so the initial NTK matrix’s structural properties (e.g., rank, conditioning) are no longer guaranteed, adding uncertainty to learning.

Prior research \citep{du2019gradient, allen2019convergence} shows overparameterized networks achieve global convergence to zero training error via Gradient Descent (GD) or Stochastic Gradient Descent (SGD) if two conditions hold:
\textit{(i)} The initial NTK matrix $\mathbf{K}_0$ is well-conditioned: its eigenvalues are bounded (no extremely large/small values to distort optimization).
\textit{(ii)} The NTK matrix $\mathbf{K}$ remains stable during training, so its structural properties do not degrade and harm convergence.

\subsection{Gradient Attenuation}\label{sec:grad}


Another advantage of proper random initialization lies in preserving an appropriate initial gradient magnitude—a factor whose critical role in the optimization process has been well validated through both experimental observations and theoretical analyses. As illustrated in \citep{2023Exit}, the time an optimizer needs to escape a saddle point is dictated by the magnitude of the initial state’s projection onto the unstable (negative-curvature) subspace of the saddle point. In consequence, excessively small initial gradients prolong the optimizer’s stagnation near saddle points, resulting in a significant deterioration in optimization performance. Unfortunately, however, gradient decay is an inherent and unavoidable phenomenon in reinforcement learning (RL) training, as evidenced by the theorem presented below.

\begin{theorem}[Gradient Dynamics at Initialization] \label{thm:initial_gradient}
For the optimization objective defined in Equation~\ref{thm:pll}, the initial gradient of the loss function (evaluated at the parameter values that minimized the loss of the previous iteration, $\hat{f}_h^{k-1}$) satisfies:
\begin{equation}
\begin{aligned}
\left. \nabla \mathbb{E}_{\mu_h^{k}} \left[ \left( f - \mathcal{T}_h \hat{f}_{h+1}^{k} \right)^2 \right] \right|_{\hat{f}_h^{k-1}} 
&=\underbrace{\frac{1}{k} \left. \nabla \mathbb{E}_{\hat{d}_h^{\pi^{k}}} \left[ \left( f - \mathcal{T}_h \hat{f}_{h+1}^{k-1} \right)^2 \right] \right|_{\hat{f}_h^{k-1}}}_{\text{Distributional shift}} \\
& \quad + \underbrace{\mathbb{E}_{\mu_h^k} \left[ \left. \nabla f^2 \right|_{\hat{f}_h^{k-1}} \cdot \left( \mathcal{T}_h \hat{f}_{h+1}^{k-1} - \mathcal{T}_h \hat{f}_{h+1}^{k} \right) \right]}_{\text{Target drift}}
\end{aligned}
\end{equation}
\end{theorem}
%
By setting $\hat{f}_{H+1} \equiv 0$. This eliminates the target-drift term entirely, leaving only the distributional-shift component—where the $\Theta(1/k)$ scaling factor becomes the dominant driver of gradient decay. As the number of training iterations $k$ grows large, the magnitude of the initial gradient will tend to approach zero. This near-zero gradient signal risks trapping the optimization process at saddle points, as the model lacks sufficient directional information to escape these suboptimal regions.


\section{Sample Weight Decay (\texttt{SWD})}

\begin{algorithm}[H]
\caption{Sample Weight Decay (\texttt{SWD})}
\label{alg:linear_decay_sampling}
\begin{algorithmic}[1]\label{algo}
\REQUIRE Linear decay steps $T$, minimum weight $w_{\min}$, Current time $t$, timestamps $\{t_i\}_{i=1}^{|\mathcal{D}|}$
    \FOR{$i = 1$ to $|\mathcal{D}|$}
        \STATE $\text{age}_i = t - t_i$ 
        \STATE $w_i = \max\left(w_{\min}, 1 - \frac{age_i}{T}\right)$ 
    \ENDFOR
    
    \STATE $p_i = \frac{w_i}{\sum_{j=1}^{|\mathcal{D}|} w_j}$ for $i = 1, \ldots, |\mathcal{D}|$ 
    \STATE $\mathcal{I} \sim \text{Categorical}(\{p_i\}_{i=1}^{|\mathcal{D}|}, B)$ 
\STATE \textbf{return} $\mathcal{B} = \{(s_i, a_i, r_i, s'_i, d_i)\}_{i \in \mathcal{I}}$
\end{algorithmic}
\end{algorithm}
\methodName is a principled algorithmic intervention to mitigate gradient signal degradation in non-stationary reinforcement learning environments. As in Algorithm~\ref{algo}, it addresses the core challenge in Theorem~\ref{thm:initial_gradient}: the harmful $\frac{1}{k}$ decay of gradient contributions from new data. It uses a linear decay mechanism, assigning each sample a weight $w_i = \max(w_{\min}, 1 - \frac{\text{age}_i}{T})$, where $\text{age}_i = t - t_i$ (t = current training step, $t_i$ = sample collection step).
The key insight of Algorithm~\ref{algo} is its rigorous sample weighting. It identifies the $\frac{1}{k}$ coefficient—overly attenuating gradients from the current policy distribution $\hat{d}^{\pi^k}$—as the root of gradient degradation. To counter this, \methodName introduces a linear weighting scheme: each sample gets a probability $p_i = \frac{w_i}{\sum_{j=1}^{|D|} w_j}$, with $p_i$ proportional to sample recency. This neutralizes the $\frac{1}{k}$ attenuation, restoring gradient magnitude and sustaining model plasticity during training.



\section{Experiments}

The core objective of the experiment is to {validate the efficacy of the proposed \methodName method in mitigating plasticity loss during long-horizon training} and quantify its performance advantages with multifaceted analyses. Specifically, we focus on the following five key research questions:

\begin{itemize}
\item \textbf{Q1}: 
Does the proposed method \methodName consistently improve the training performance of mainstream reinforcement learning (RL) algorithms across different continuous and discrete control tasks? 
    
    \item \textbf{Q2}: Does the temporal weighting strategy of the proposed method \methodName play a critical role in alleviating plasticity loss?
    
    \item \textbf{Q3}: Can \methodName adapt to the training scenarios with increased Update-to-Data (UTD) ratio configurations, where more severe plasticity loss should be addressed for better data efficiency?
    
    \item \myadded{\textbf{Q4}: How does \methodName compare with other methods designed to address plasticity issues? And is it feasible to combine \methodName with these other methods?}
    
    \item \myadded{\textbf{Q5}: How sensitive is the proposed \methodName to the hyperparameters?
    How do different choices of heuristics influence the results?
    }
\end{itemize}
To address Q1, we conduct experiments using the \myadded{Double DQN}, TD3, and SAC algorithms within the SimBa architecture~\citep{lee2025simba}, evaluating their performance across the \myadded{Arcade Learning Environment}~\citep{ale}, the MuJoCo environments \citep{MuJoCo}, and the DMC suite \citep{dmc}. We also include the canonical method Prioritized Experience Replay (PER)~\citep{PER} as a direct baseline method. Furthermore, to provide reverse validation of \methodName's effectiveness, we use a variant called Sample Weight Augmentation (\texttt{SWA}), i.e., a counterpart designed to produce the opposite effect by assigning higher weights to older samples. 
For Q2, we adopt \texttt{GraMa}~\citep{grama} as the metric for plasticity, using it to empirically demonstrate the superiority of our proposed method in alleviating plasticity loss. To answer Q3, we evaluate the performance of \methodName based on Simba-SAC under different UTD ratios, with a specific focus on the Humanoid Run environment. \myadded{To address Q4,  we compare \methodName against other representative methods designed to address plasticity issues. For Q5, we conduct extensive experiments to analyze the hyperparameter sensitivity of \methodName and the effects of different decay strategies, such as exponential decay and polynomial decay.
}

\subsection{Performance Evaluation}\label{sec:main_results}
\begin{figure}[t]
\centering  
\subfigure[Ant]{
\label{Fig:ant}
\includegraphics[width=0.3\textwidth]{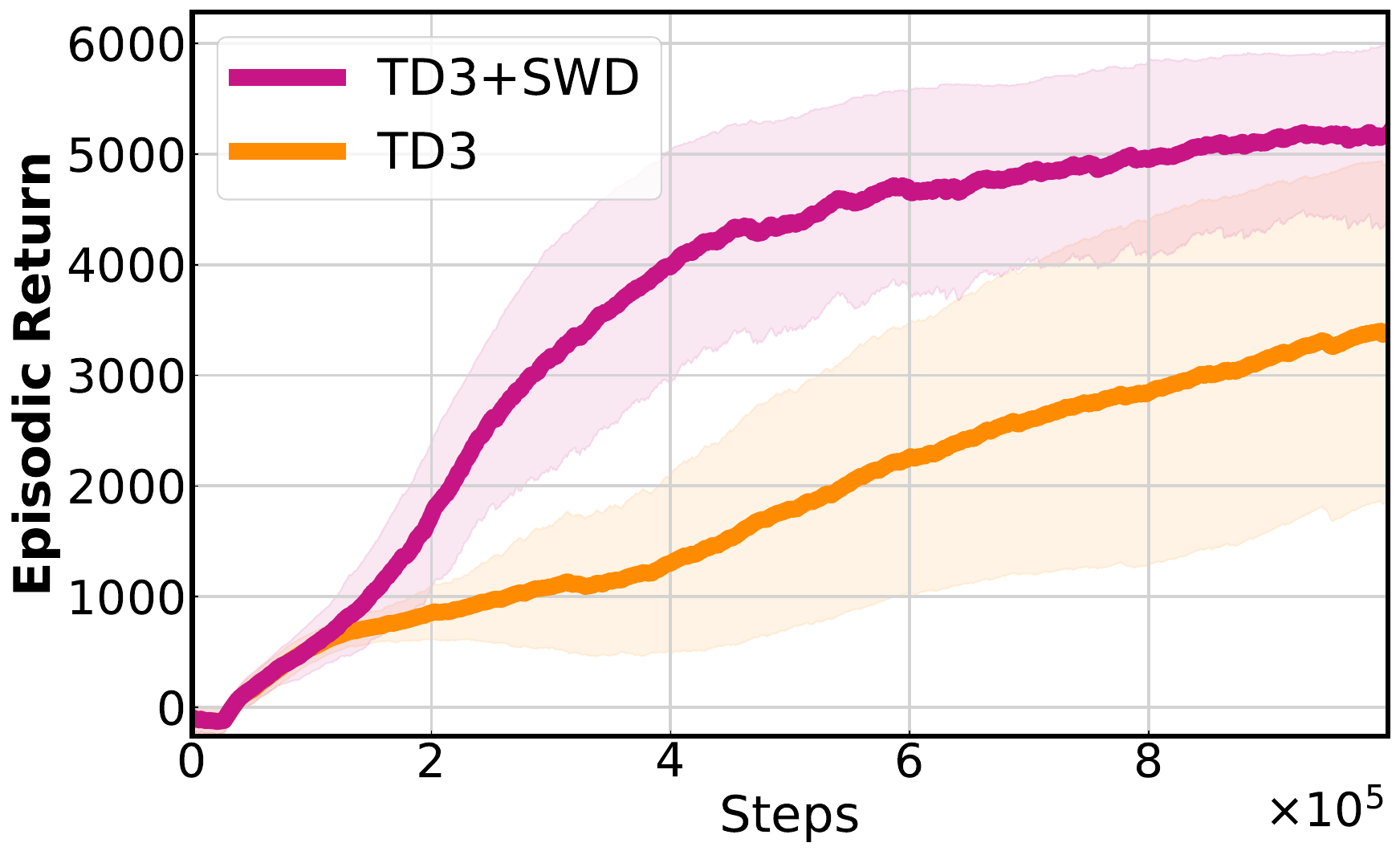}
}
\subfigure[Halfcheetah]{
\label{Fig:Halfcheef}
\includegraphics[width=0.3\textwidth]{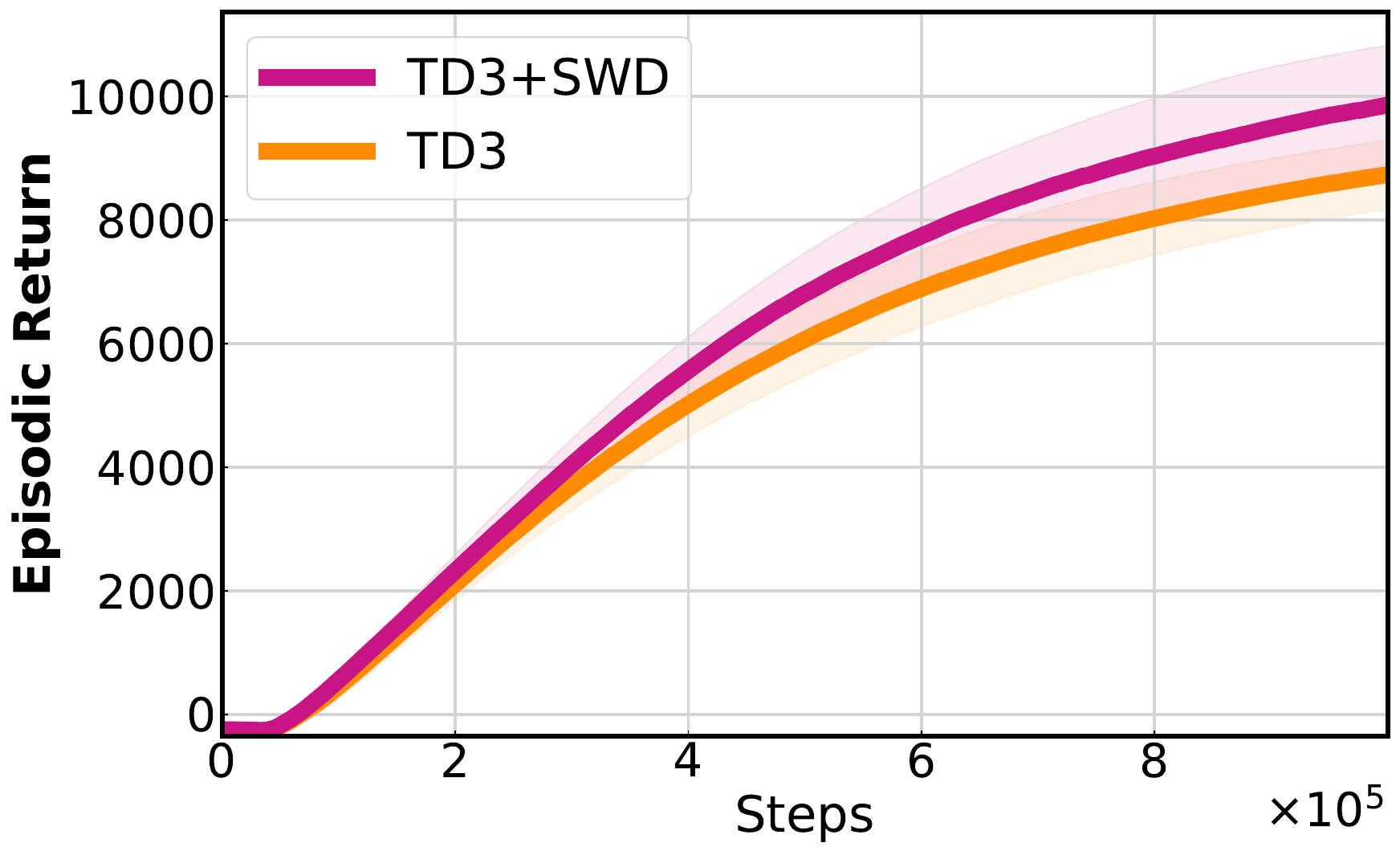}}
\subfigure[Humanoid]{
\label{Fig:Humanoid}
\includegraphics[width=0.3\textwidth]{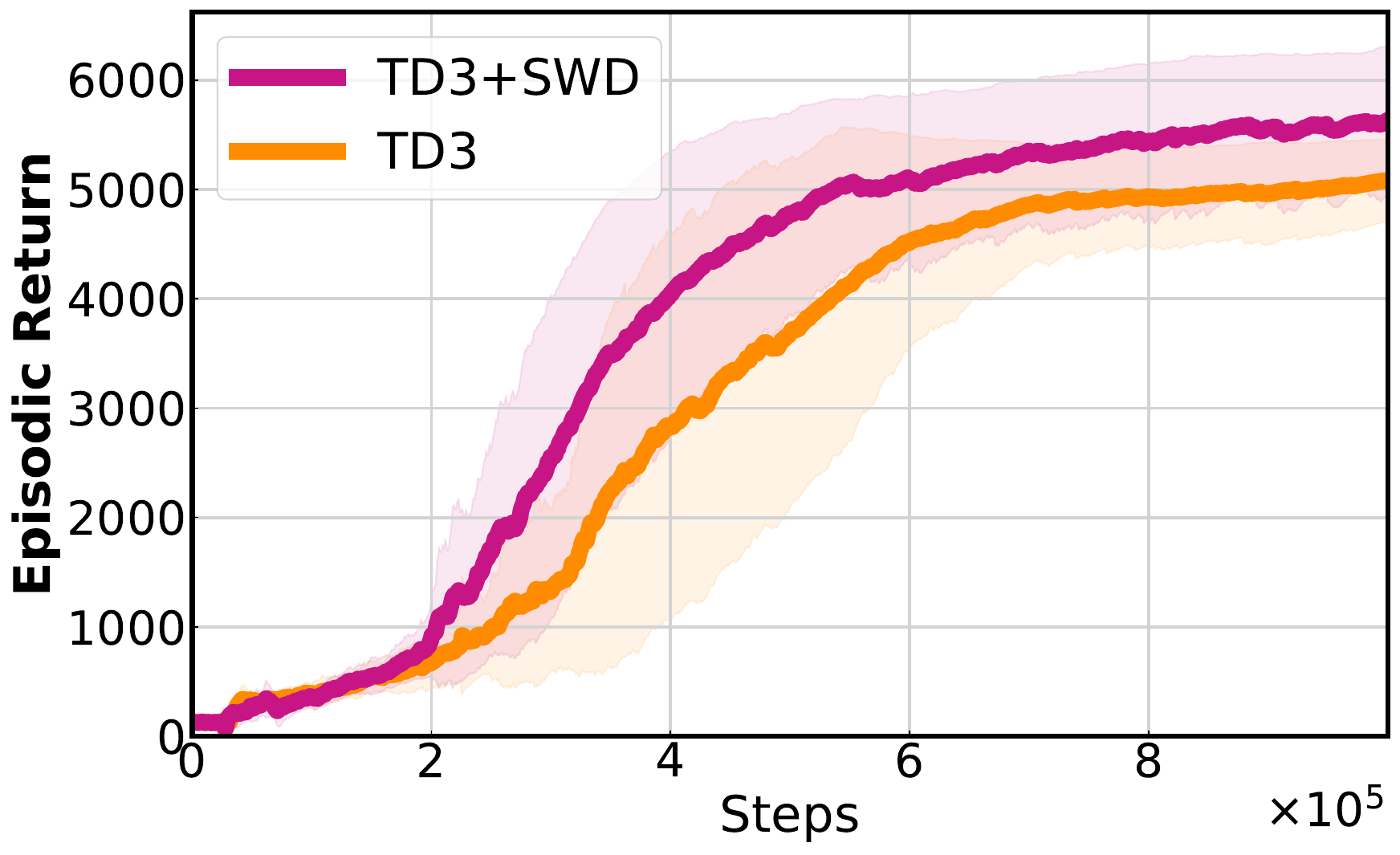}
}
\subfigure[Walker2d]{
\label{Fig:Walker2d}
\includegraphics[width=0.3\textwidth]{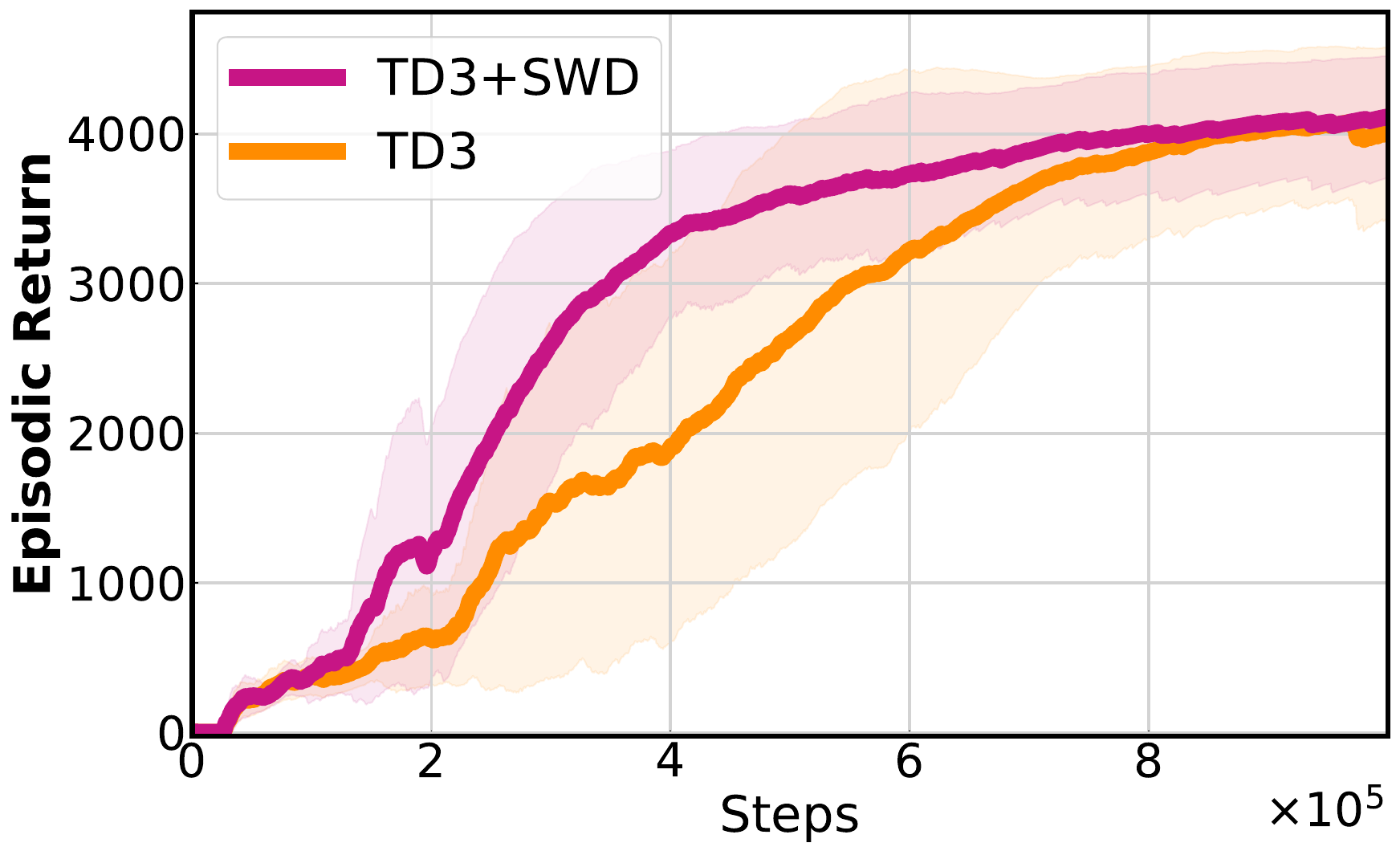}
}
\subfigure[Hopper]{
\label{Fig:Hopper}
\includegraphics[width=0.3\textwidth]{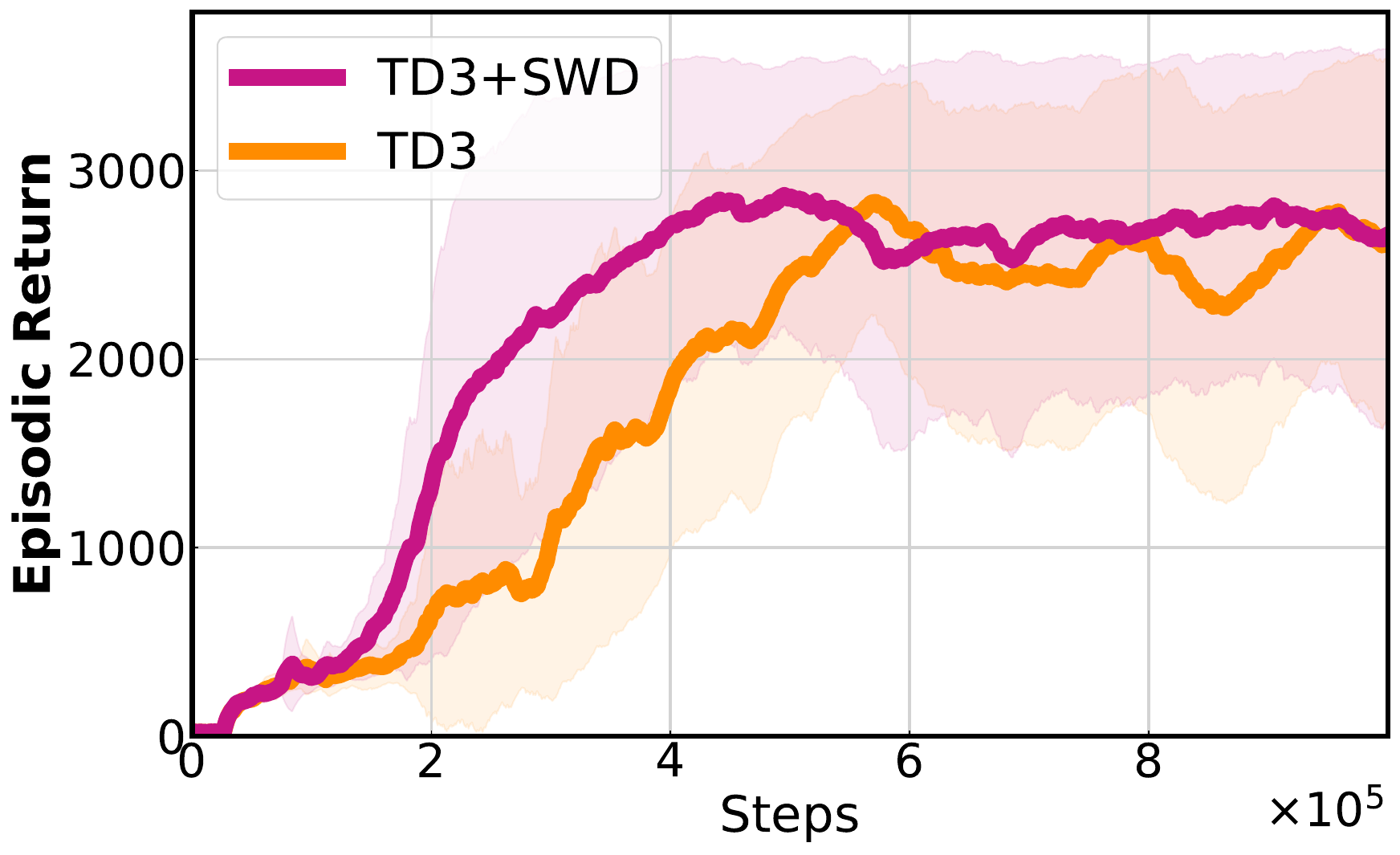}
}
\vspace{-0.3cm}
\caption{Empirical validation of \texttt{SWD} across TD3 in MuJoCo environments (mean $\pm$ std over 5 runs). \texttt{SWD} consistently improves sample efficiency and performance.}
\label{Fig:MuJoCo}
\vspace{-0.3cm}
\end{figure}

\begin{figure}[t]
\centering  
\subfigure[Phoenix]{
\label{Fig:Phoenix}
\includegraphics[width=0.3\textwidth]{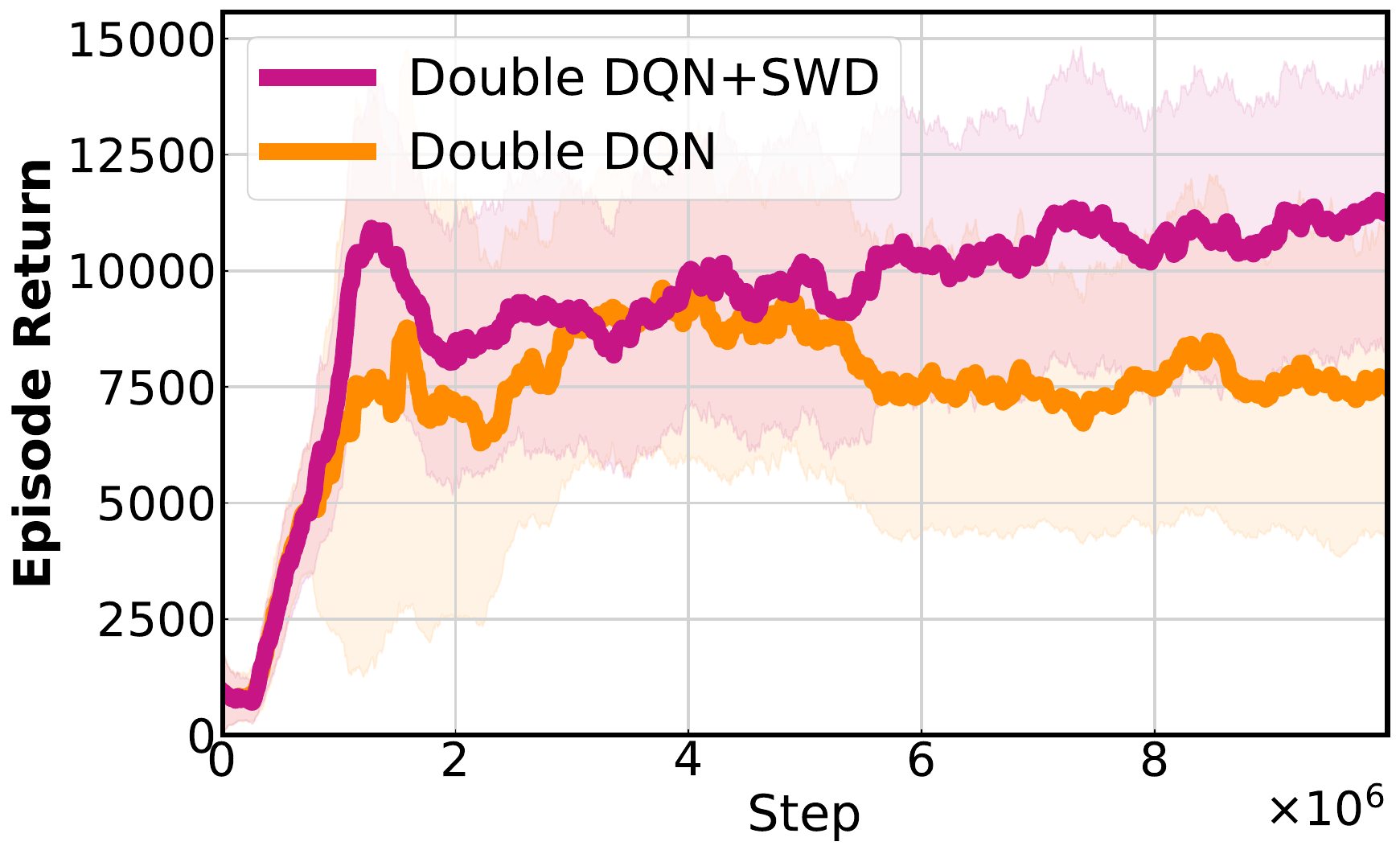}
}
\subfigure[DemonAttack]{
\label{Fig:DemonAttack}
\includegraphics[width=0.3\textwidth]{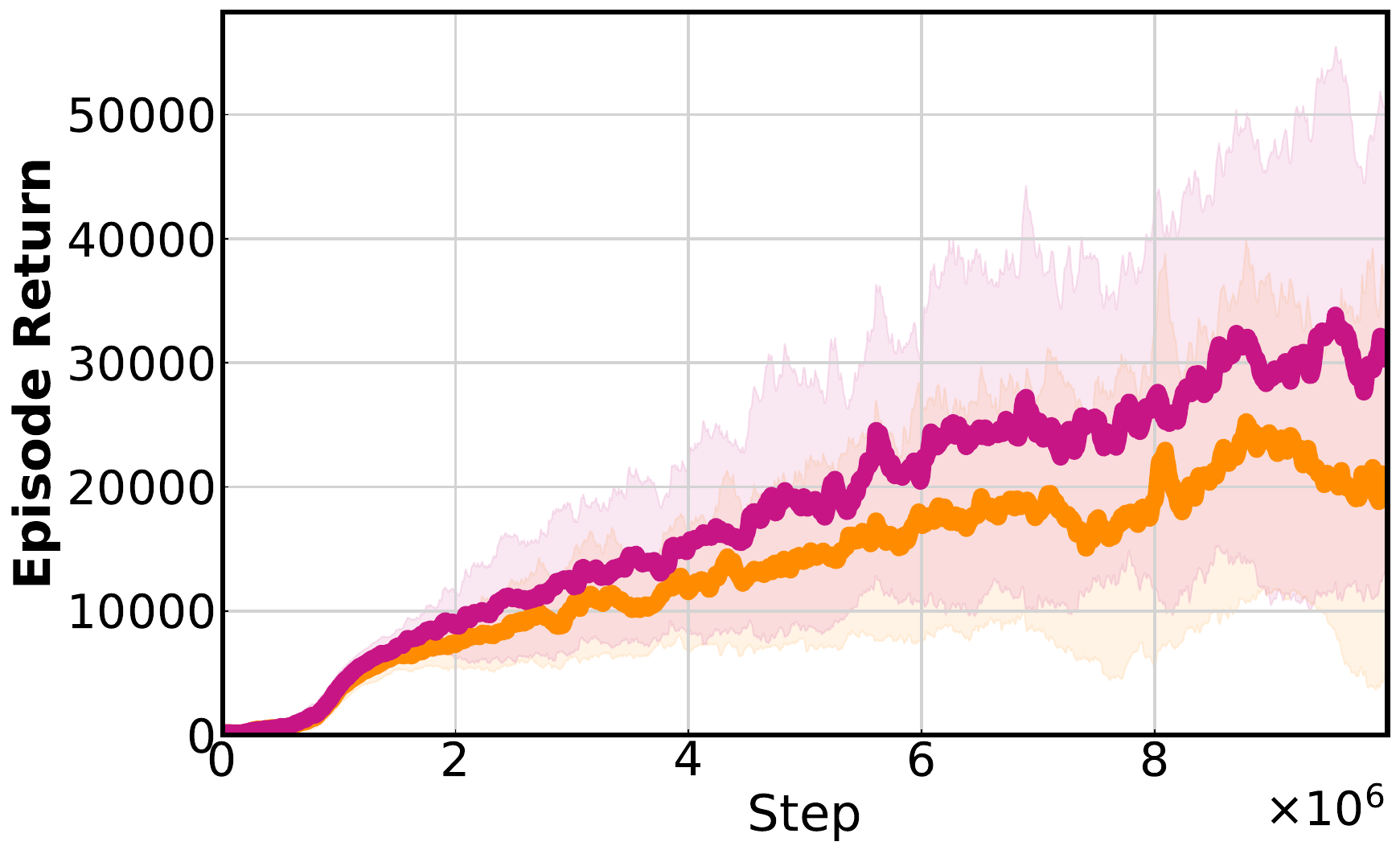}}
\subfigure[Breakout]{
\label{Fig:Breakout}
\includegraphics[width=0.3\textwidth]{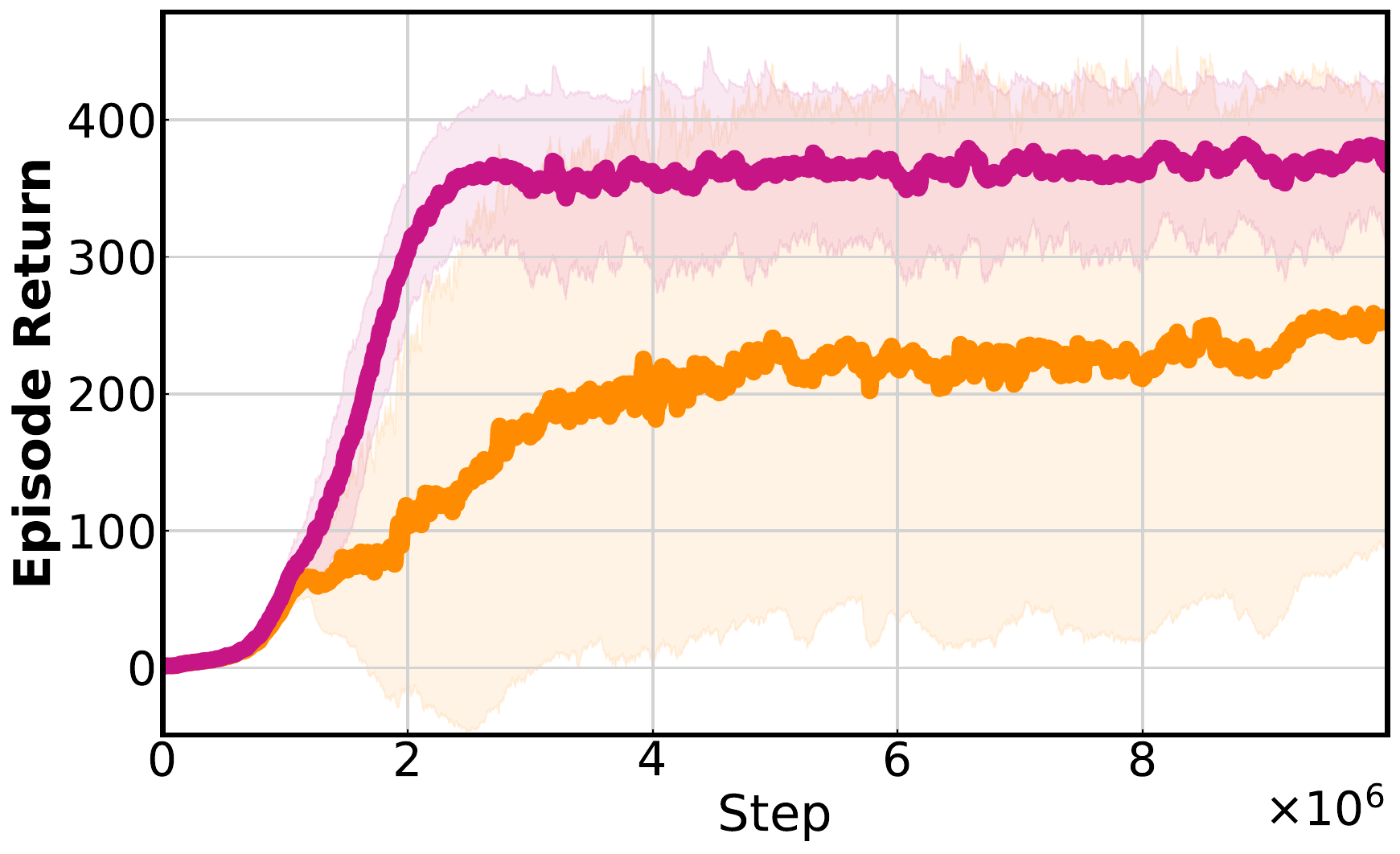}
}
\vspace{-0.3cm}
\caption{\myadded{Empirical validation of \texttt{SWD} across Double DQN in ALE environments (mean $\pm$ std over 5 runs). \texttt{SWD} consistently improves sample efficiency and performance.}}
\label{Fig:ALE}
\vspace{-0.3cm}
\end{figure}

\paragraph{Experimental Setup.} We evaluate methods on three benchmark suites: 
\textit{(i)} For the five MuJoCo environments (Ant, HalfCheetah, Hopper, Humanoid, Walker2d), we use TD3~\citep{td3} as the base algorithm with conventional MLP networks. 
\myadded{\textit{(ii)} For the three ALE environments (DemonAttack, Phoenix, and Breakout), we use Double Deep Q-Network~\citep{DDQN} as the base algorithm with the typical CNN-MLP networks. }
\textit{(iii)} For the four difficult DMC tasks (Humanoid-Run, Humanoid-Walk, Dog-Run, Dog-Walk),we use SAC~\citep{sac} as the base algorithm with the SimBa network architecture~\citep{lee2025simba}. 
In this subsection, we include {PER} as a canonical baseline for comparison. Detailed hyperparameters and details are provided in Appendix ~\ref{App:exp}.

\vspace{-0.2cm}
\paragraph{Results} 
As illustrated in Figure~\ref{Fig:MuJoCo}, \myadded{Figure}~\ref{Fig:ALE} and Figure~\ref{fig:rli}, \methodName demonstrates a remarkable ability to enhance the algorithm's performance. Specifically, it facilitates accelerated learning during the early phases of training and attains superior final policy quality upon convergence—an advantage that is particularly prominent in the Ant and Humanoid environments. In sharp contrast, PER (Prioritized Experience Replay) demands nearly several times more training time, while the performance improvements it yields remain extremely limited. This observation aligns well with our theoretical framework, and Equation~\ref{eq:bound} further confirms that performance enhancement can only be achieved by optimizing the TD errors along both the optimal policy path and the current policy path.

\begin{figure}
    \centering\includegraphics[width=1\linewidth]{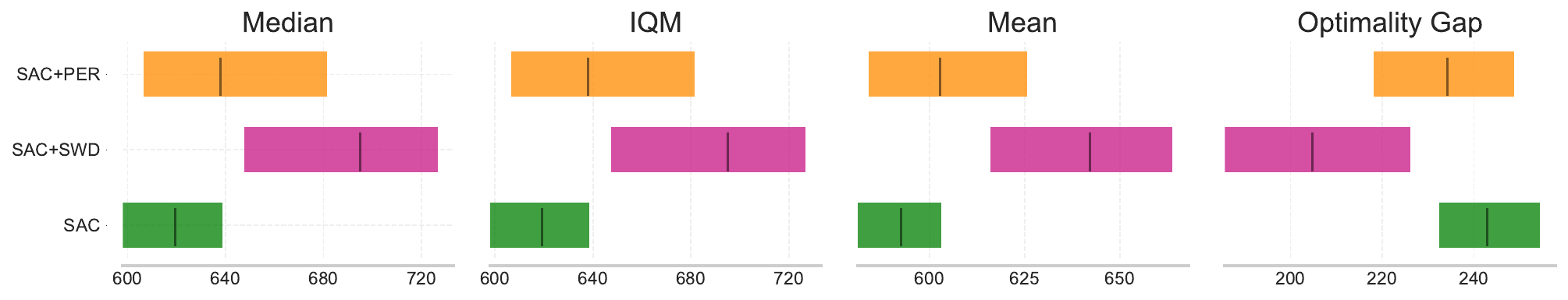}
    \vspace{-0.7cm}\caption{{Performance comparison between \texttt{SWD} and PER based on SAC.} Aggregate \textit{Reliable} metrics~\citep{agarwal2021deep} with 95\% Stratified Bootstrap CIS in DMC tasks.}
    \label{fig:rli}
    \vspace{-0.2cm}
\end{figure}


\subsection{Ablation Study}
To provide reverse validation of \texttt{SWD}'s effectiveness, we develop a contrasting method called{Sample Weight Augmentation (\texttt{SWA})}, which implements the opposite weighting strategy by assigning higher weights to older data samples. This design allows us to empirically verify our theoretical hypothesis that prioritizing recent experiences is crucial for maintaining neural plasticity.More details are shown in Appendix~\ref{App:additional exp}, where we employed GraMa~\citep{grama} as our measure of neural plasticity.

\begin{figure}[htbp]
\centering  
\subfigure[Performance Comparison]{
\label{Fig:swa_performance}
\includegraphics[width=0.3\textwidth]{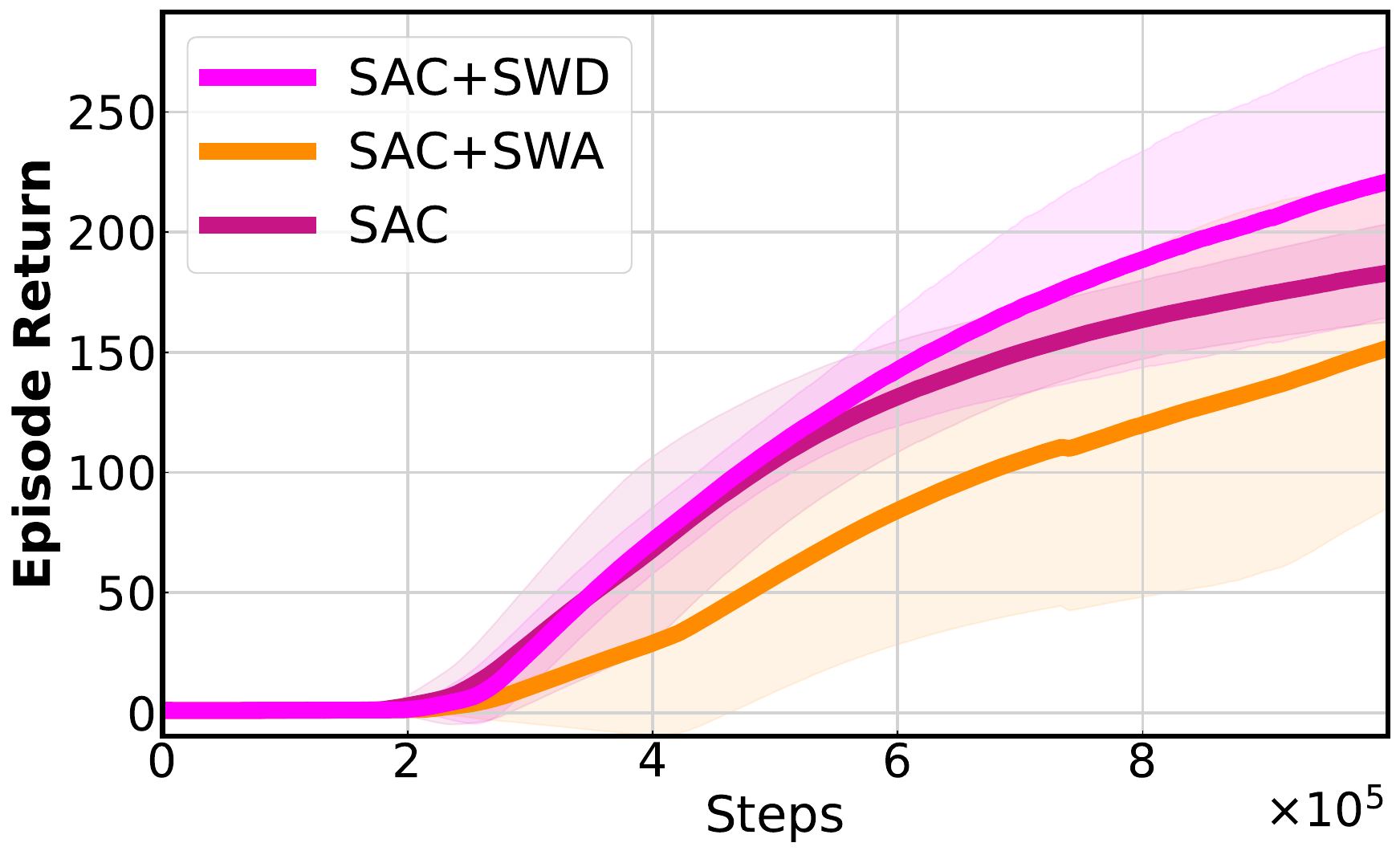}
}
\subfigure[Gradient L1 Norm Evolution]{
\label{Fig:l2_norm}
\includegraphics[width=0.3\textwidth]{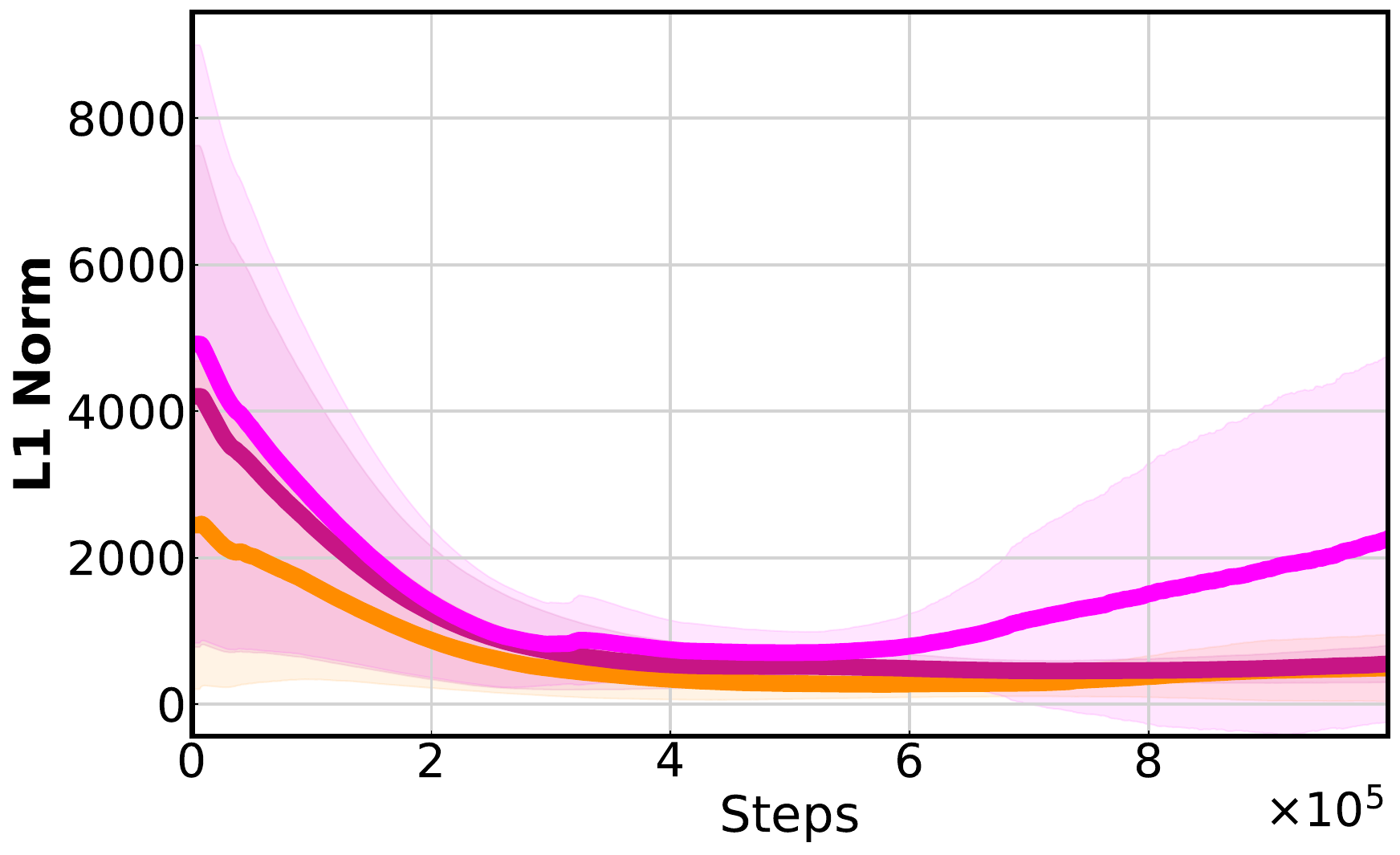}
}
\subfigure[GraMa]{
\label{Fig:grad_sparsity}
\includegraphics[width=0.3\textwidth]{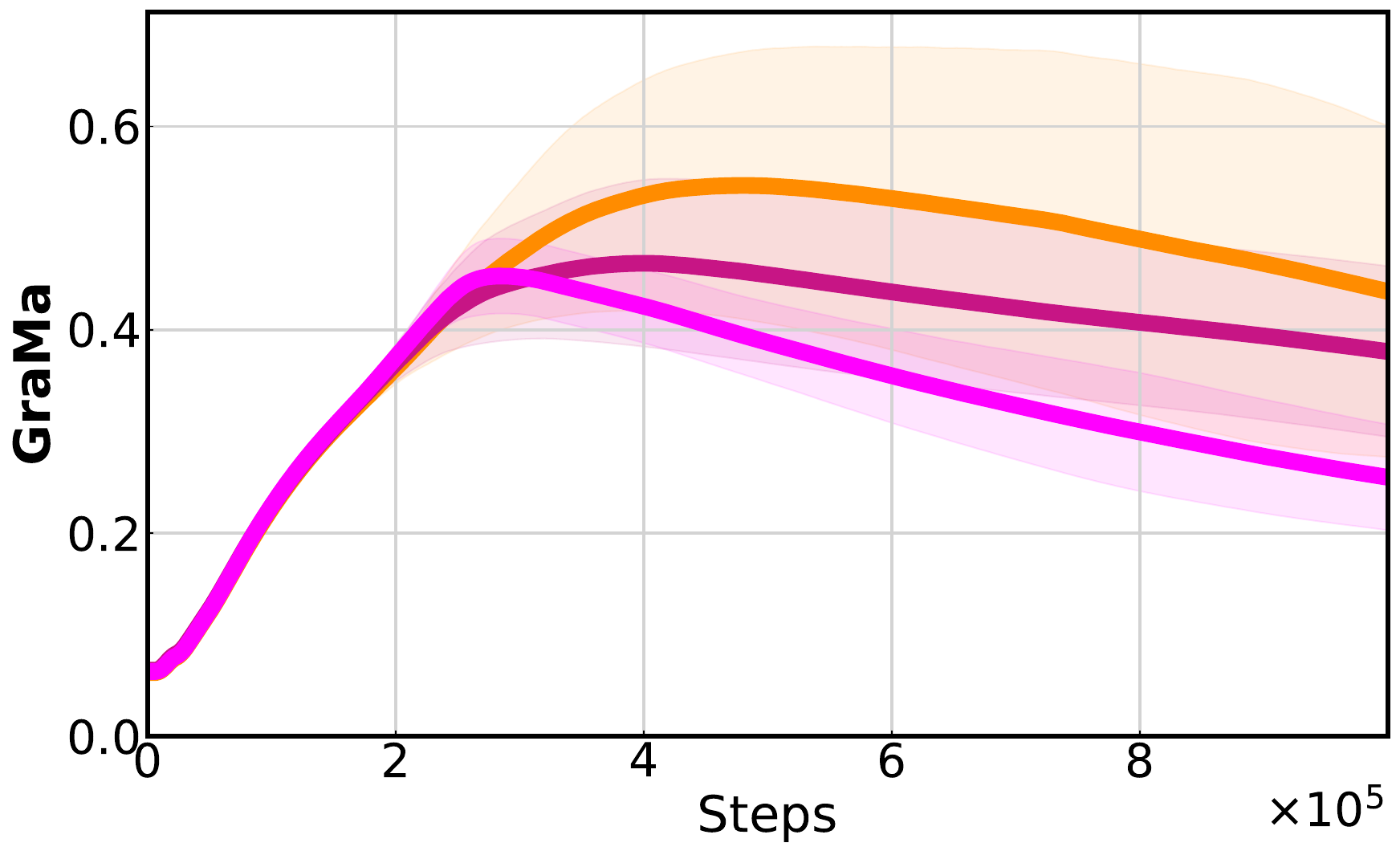}
}
\vspace{-0.3cm}
\caption{Experiments conducted in the humanoid-run environment demonstrate that \texttt{SWA} exhibits a lower gradient magnitude, GraMa, and inferior performance, which validates our hypothesis.}
\label{Fig:SWA}
\vspace{-0.2cm}
\end{figure}

\vspace{-0.2cm}
\paragraph{Results} The reverse validation experiment yields key insights: \textit{(i)} As shown in Figure~\ref{Fig:swa_performance}, \texttt{SWA} consistently underperforms \texttt{SWD} and uniform sampling, validating that prioritizing recent experiences is critical for non-stationary RL learning; \textit{(ii)} Figure~\ref{Fig:l2_norm} shows \texttt{SWA} reduces gradient L1 norms during training (weakened learning signals), aligning with our gradient attenuation analysis and confirming older data exacerbates plasticity loss; \textit{(iii)} GraMa analysis in Figure~\ref{Fig:grad_sparsity} reveals \texttt{SWA} causes sparser gradients and greater plasticity loss than \texttt{SWD} (reduced neural activation/adaptation capacity), providing direct empirical support for our theoretical framework’s plasticity degradation prediction.

\subsection{The Effect in Alleviating Plasticity Loss}\label{exp:plasticity}

\begin{figure}
    \centering
    \subfigure[Humanoid Run]{
    \label{Fig:humanoid run}
    \includegraphics[width=0.3\textwidth]{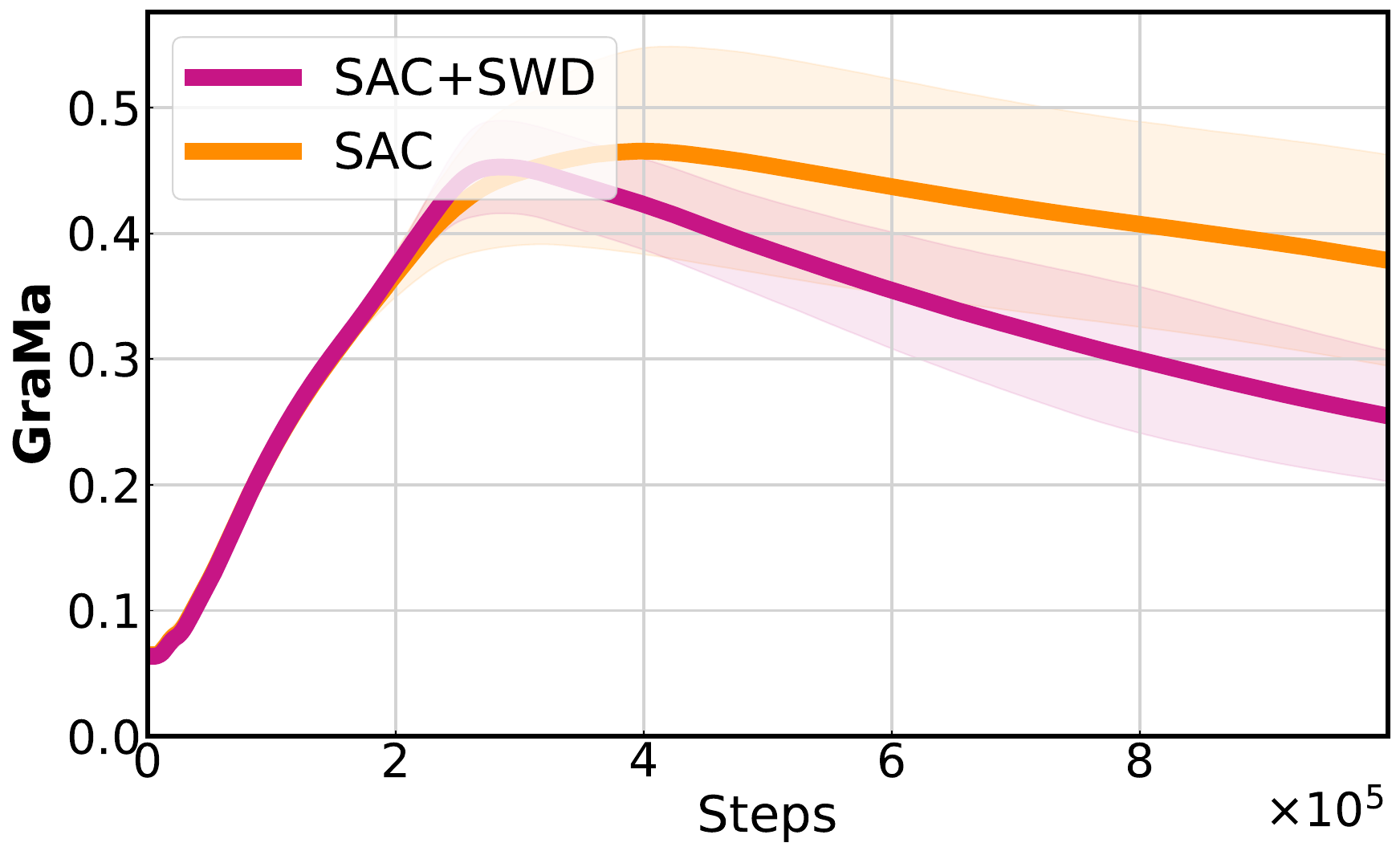}
    }
    \subfigure[Humanoid Walk]{
    \label{Fig:humanoid walk}
    \includegraphics[width=0.3\textwidth]{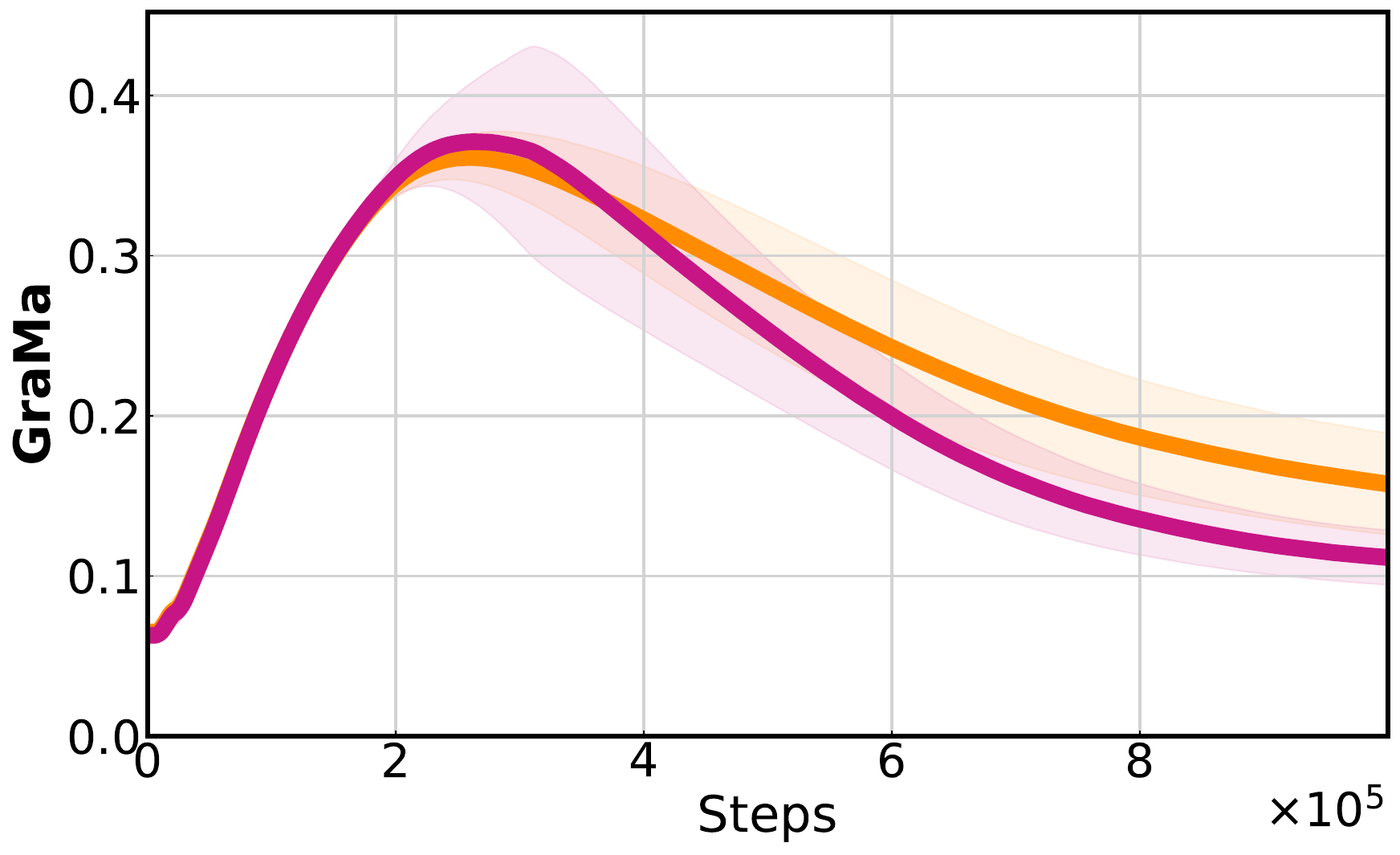}
    }
    \subfigure[Humanoid Stand]{
    \label{Fig:humanoid stand}
    \includegraphics[width=0.3\textwidth]{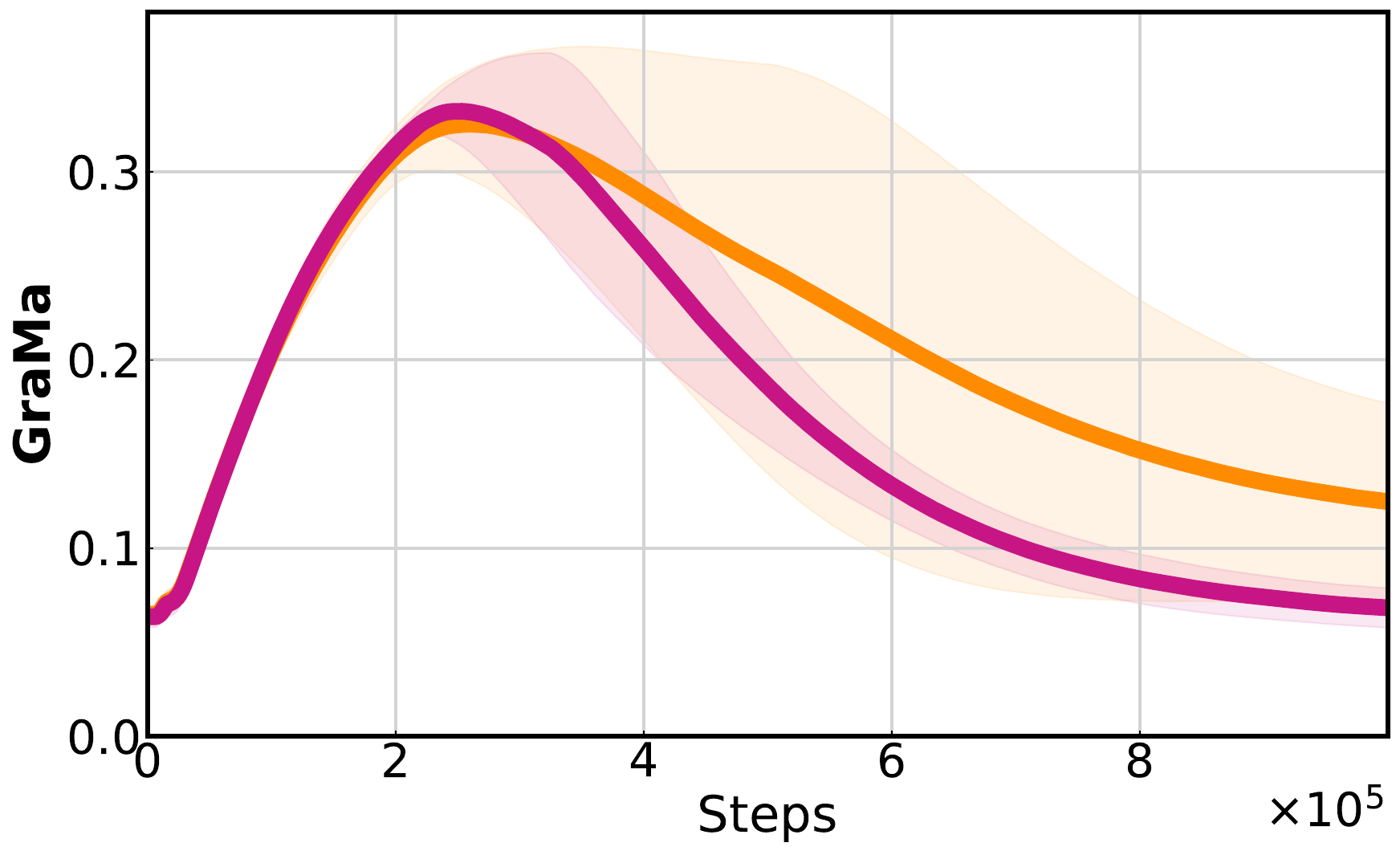}
    }
    \vspace{-0.3cm}
    \caption{\texttt{GraMa} Metric in Humanoid Locomotion: Run, Walk, and Stand: The results clearly demonstrate that \methodName effectively mitigates the loss of plasticity in humanoid robots across these key locomotor states.}
\label{Fig:grama}
    \vspace{-0.3cm}
\end{figure}

To verify whether \methodName can mitigate plasticity, we employed \texttt{GraMa} as the evaluation metric to quantify the degree of plasticity during the model training process. Notably, a larger \texttt{GraMa} value indicates a weaker learning capability of the neural network.

\vspace{-0.2cm}
\paragraph{Results} The corresponding results are illustrated in Figure~\ref{Fig:grama}. As depicted in this figure, our proposed \methodName effectively alleviates the gradient sparsity that arises during the training process. Notably, the most pronounced effects are observed in the Humanoid Run environment and the Humanoid Stand environment. It can be clearly seen from the figure that \methodName exerts its function in the middle and late stages of training --- gradient attenuation is not severe in the early stage --- and this observation is consistent with our theoretical predictions.

\subsection{Compatibility Against Higher Update-to-Data Ratios}\label{sec:UTD}

\begin{figure}[htbp]
    \centering
    \includegraphics[width=0.7\linewidth]{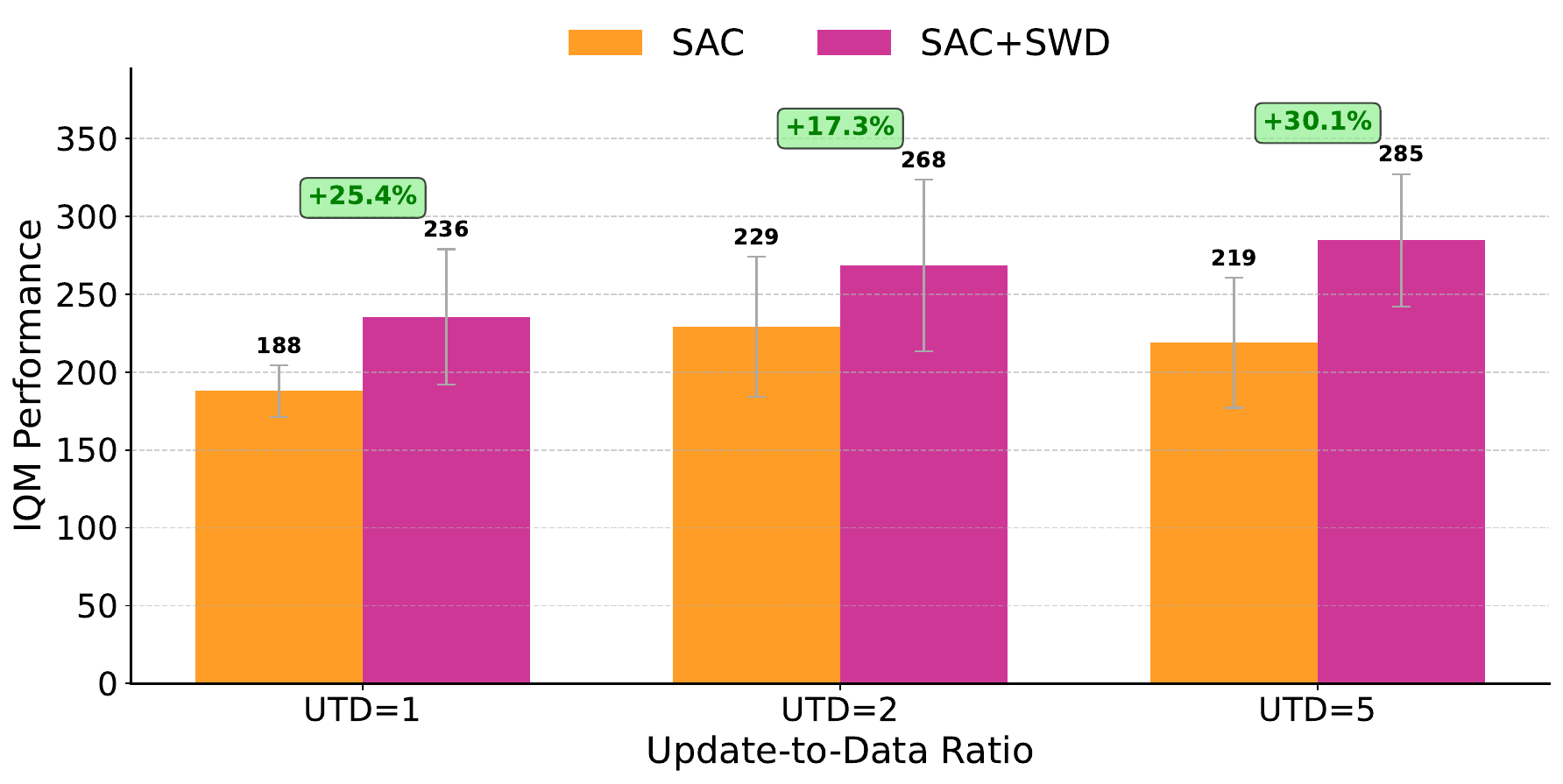}
    \vspace{-0.3cm}
    \caption{Performance comparison across different UTD ratios (1, 2, 5) in Humanoid Run. \texttt{SWD} consistently outperforms uniform sampling across all UTD settings, with improvements ranging from 17.3\% to 30.1\%.}
    \label{fig:utd_sensitivity}
\end{figure}

The Update-to-Data Ratio (UTD) is a critical metric for measuring an algorithm’s data utilization efficiency. Intuitively, uniform sampling assigns equal weight to each sample; after multiple updates, the gradient signals that can effectively guide the update of network parameters become very weak. In contrast, our \texttt{SWD} method assigns greater weight to more recent samples, ensuring that sufficiently strong gradient signals are maintained even after multiple updates.

As shown in Figure~\ref{fig:utd_sensitivity}, \texttt{SWD} demonstrates consistent effectiveness across UTD ratios of 1, 2, and 5. Notably, the method shows the largest improvement (+30.1\%) at UTD=5, suggesting that \texttt{SWD} is particularly beneficial when gradient updates are frequent. This robustness indicates that our approach is broadly applicable across different algorithmic configurations without requiring UTD-specific tuning.

\begin{figure}
\vspace{-0.3cm}
    \centering
      \includegraphics[width=0.9\textwidth]{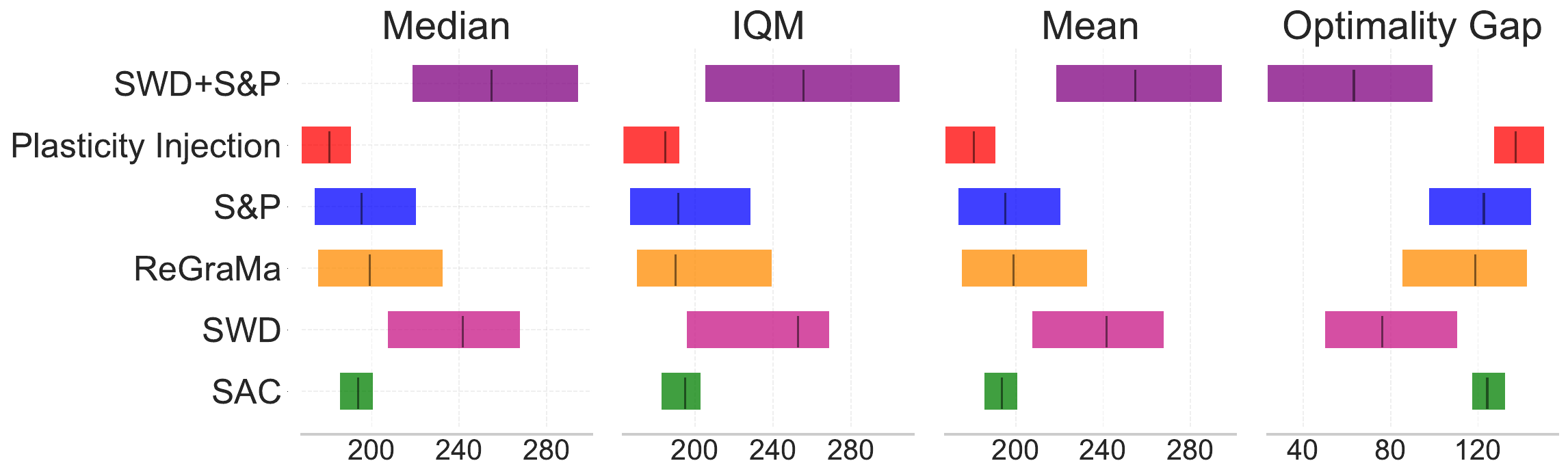}
\vspace{-0.3cm}
\caption{\myadded{Performance comparison between \texttt{SWD+S\&P} and other methods designed to address plasticity issues. Aggregate \textit{Reliable} metrics~\citep{agarwal2021deep} with 95\% Stratified Bootstrap CIS in Humanoid run.}}
    \label{Fig:compare}
\vspace{-0.3cm}
\end{figure}

\subsection{Comparison with Other Methods Designed to Address Plasticity Loss}\label{sec:plas_comp}
\myadded{To further evaluate the effectiveness of \methodName, we compare it in the Humanoid Run environment with three representative methods that are designed to address plasticity issues: {ReGraMa}~\citep{grama}, {S\&P}~\citep{s&p}, and {Plasticity Injection}~\citep{plasticityinjection}. Moreover, we explore \methodName's synergistic potential with S\&P, i.e., \texttt{SWD+S\&P}, which demonstrates the orthogonality.}

\vspace{-0.7cm}
\myadded{
\paragraph{Results} As in Figure~\ref{Fig:compare}, \methodName outperforms other NTK-based methods on the SimBa~\citep{lee2025simba} network. Moreover, \methodName combined with S\&P yields the best result, validating its orthogonality to NTK-based methods. We provide a detailed discussion on the relationship between \methodName and prior works for plasticity loss in Appendix~\ref{app:relationship}.
}

\subsection{Other Results}\label{subsec:hyperparam_others}
\vspace{-1cm}
\myadded{
\paragraph{Hyperparameter Choices and Decay Strategies}
To analyze the hyperparameter sensitivity of \methodName, we conduct a grid-search test for two core hyperparameters, i.e., linear decay steps $T$ and minimum weight threshold $w_{min}$. In Table~\ref{tab:sensitivity_clean} of Appendix~\ref{App:additional exp}, \methodName exhibits low sensitivity to different choices, demonstrating its stability.
Moreover, we compare the linear decay strategy of \methodName with the other two commonly adopted strategies, i.e., exponential decay and polynomial decay. Table~\ref{tab:decay_strategies} shows that the linear decay strategy outperforms the other two strategies.}

\vspace{-0.7cm}
\myadded{
\paragraph{Compute-efficient Approximation of \methodName} 
To further reduce the computational overhead of per-sample weight, we propose a bucket-based approximation method. As in Table~\ref{tab:runtime_perf} of Appendix~\ref{app:approix}, this approximation significantly reduces the training time at no compromise of policy performance.
}



\vspace{-0.1cm}
\section{Conclusion}
\vspace{-0.3cm}
In this paper, we identified and addressed the critical issue of plasticity loss in long-horizon reinforcement learning through both theoretical analysis and algorithmic innovation. Our theoretical framework reveals that gradient attenuation follows a $\Theta(1/k)$ decay pattern, fundamentally limiting the agent's ability to adapt to new experiences over extended training periods.
To counteract this degradation, we proposed Sample Weight Decay (\texttt{SWD}), a simple yet effective method that applies age-based weighting to replay buffer sampling. Through comprehensive experiments across MuJoCo, ALE and DMC environments with TD3, DDQN and SAC algorithms, we demonstrated consistent performance improvements ranging from 13.7\% to 30.1\% in IQM scores. Our ablation studies and reverse validation experiments confirm that temporal weighting direction is crucial for maintaining neural plasticity.
The broad applicability of \texttt{SWD} across different algorithms, environments, and training configurations, combined with its minimal computational overhead, makes it a practical solution for enhancing long-horizon RL performance. This work opens new avenues for understanding and mitigating plasticity loss in deep reinforcement learning.

\paragraph{Limitations} Owing to computational constraints, our evaluation is restricted to tasks within the MuJoCo, ALE and DeepMind Control Suite (DMC). Additionally, our exploration and practical application of the proposed theoretical framework remain at a preliminary stage—representing merely the "tip of the iceberg." Moving forward, future research will extend \methodName to more complex scenarios, real-world environments. Ultimately, our goal is to develop \methodName into a practical, robust tool that effectively preserves the learning capacity of deep reinforcement learning (RL) agents.



\newpage
\section*{Acknowledgments}
This work is supported by the National Natural Science Foundation of China (Grant Nos. 62422605, 62533021, 62541610, 92370132), the Fundamental Research Program of Shanxi Province (No. 202503021212091), and the National Key Research and Development Program of China (Grant No. 2024YFE0210900).
\bibliography{iclr2026_conference}
\bibliographystyle{iclr2026_conference}

\newpage
\appendix
\onecolumn
\section{The Use of Large Language Models (LLMs)}
Large Language Models (LLMs) were utilized to support the writing and refinement of this manuscript. Specifically, an LLM was employed to assist in enhancing language clarity, improving readability, and ensuring coherent expression across different sections of the paper. It aided in tasks like rephrasing sentences, checking grammar, and optimizing the overall textual flow.

It should be noted that the LLM played no role in the conception of research ideas, the formulation of research methodologies, or the design of experiments. All research concepts, ideas, and analyses were independently developed and carried out by the authors. The LLM's contributions were strictly limited to elevating the linguistic quality of the paper, without any involvement in the scientific content or data analysis.

The authors fully assume responsibility for the entire content of the manuscript, including any text generated or polished with the help of the LLM. We have verified that the text produced with the LLM complies with ethical guidelines and does not lead to plagiarism or any form of scientific misconduct.

\section{Proof}
\subsection{Proof of Theorem~\ref{thm:pll}}In this section, we prove Theorem\ref{thm:pll}.

\begin{equation*}
\begin{aligned}
    &\mathbb{E} \mathcal{L}_h^k(f,\hat{f}_{h+1}^k) \\
    &= \mathbb{E}_{(s_h,a_h)\sim \mu_h^k}\left[\mathbb{E}_{s_{h+1}\sim p_h(\cdot\mid s_h,a_h)}\left[\left(f(s_h,a_h)-r(s_h,a_h)-\max_{a'}\hat{f}_{h+1}^k(s_{h+1},a')\right)^2\right]\right]\\[0.5em]
    &= \mathbb{E}_{(s_h,a_h)\sim \mu_h^k}\Bigg[\mathbb{E}_{s_{h+1}\sim p_h(\cdot\mid s_h,a_h)}\bigg[\Big(f(s_h,a_h)-\mathcal{T}_h\hat{f}_{h+1}^k(s_h,a_h)\\
    &\qquad\qquad\qquad\qquad\qquad\qquad\qquad + \mathbb{P}_h\max_{a'}\hat{f}_{h+1}^k(s_{h+1},a')(s_h,a_h)-\max_{a'}\hat{f}_{h+1}^k(s_{h+1},a') \Big)^2\bigg]\Bigg]\\[0.5em]
    &= \mathbb{E}_{(s_h,a_h)\sim \mu_h^k}\Bigg[\mathbb{E}_{s_{h+1}\sim p_h(\cdot\mid s_h,a_h)}\bigg[\left(f(s_h,a_h)-\mathcal{T}_h\hat{f}_{h+1}^k(s_h,a_h)\right)^2\\
    &\qquad\qquad\qquad\qquad\qquad\qquad\qquad + \left(\mathbb{P}_h\max_{a'}\hat{f}_{h+1}^k(s_{h+1},a')(s_h,a_h)-\max_{a'}\hat{f}_{h+1}^k(s_{h+1},a')\right)^2\\
    &\qquad\qquad\qquad\qquad\qquad\qquad\qquad + 2\left(f(s_h,a_h)-\mathcal{T}_h\hat{f}_{h+1}^k(s_h,a_h)\right)\\
    &\qquad\qquad\qquad\qquad\qquad\qquad\qquad\qquad \times \left(\mathbb{P}_h\max_{a'}\hat{f}_{h+1}^k(s_{h+1},a')(s_h,a_h)-\max_{a'}\hat{f}_{h+1}^k(s_{h+1},a')\right)\bigg]\Bigg]\\[0.5em]
    &= \mathbb{E}_{(s_h,a_h)\sim \mu_h^k}\left[\left(f(s_h,a_h)-\mathcal{T}_h\hat{f}_{h+1}^k(s_h,a_h)\right)^2\right]\\
    &\quad + \mathbb{E}_{(s_h,a_h)\sim \mu_h^k}\left[\mathbb{E}_{s_{h+1}\sim p_h(\cdot\mid s_h,a_h)}\left[\left(\mathbb{P}_h\max_{a'}\hat{f}_{h+1}^k(s_{h+1},a')(s_h,a_h)-\max_{a'}\hat{f}_{h+1}^k(s_{h+1},a')\right)^2\right]\right]\\
    &= \mathbb{E}_{(s_h,a_h)\sim \mu_h^k}\left[\left(f(s_h,a_h)-\mathcal{T}_h\hat{f}_{h+1}^k(s_h,a_h)\right)^2\right] + \mathbb{E}_{(s_h,a_h)\sim \mu_h^k}\left[\text{Var}_{s_{h+1}\sim p_h(\cdot\mid s_h,a_h)}\left[\max_{a'}\hat{f}_{h+1}^k(s_{h+1},a')\right]\right]
\end{aligned}
\end{equation*}

The loss function can be decomposed into two components:
\begin{itemize}
\item \textbf{Bellman Residual Term}: $\mathbb{E}_{(s_h,a_h)\sim \mu_h^k}\left[\left(f(s_h,a_h)-\mathcal{T}_h\hat{f}_{h+1}^k(s_h,a_h)\right)^2\right]$--Measures the function approximation error.
\item \textbf{Environmental Stochasticity Term}: $\mathbb{E}_{(s_h,a_h)\sim \mu_h^k}\left[\text{Var}_{s_{h+1}\sim p_h(\cdot\mid s_h,a_h)}\left[\max_{a'}\hat{f}_{h+1}^k(s_{h+1},a')\right]\right]$--Reflects the intrinsic randomness of state transitions.
\end{itemize}

\subsection{Proof of Theorem~\ref{thm:subopt-f}}
First, we prove one lemma to help the proof.
\begin{lemma}
Consider an episodic MDP with horizon $H$. Let $\pi' = \{\pi'_h\}_{h=1}^H$ denote any policy, and let $\{\hat{Q}_h\}_{h=1}^H$ denote any set of estimated Q-functions. Let $\pi = \{\pi_h\}_{h=1}^H$ be the greedy policy induced by $\{\hat{Q}_h\}_{h=1}^H$.

For all $h \in [H]$, define:
\begin{itemize}
\item Value function: $\hat{V}_h(s) = \mathbb J_h^{\pi}\hat{Q}_h(s)$ where $\mathbb J_h^{\pi}f(s) = \mathbb E_{a \sim \pi_h(\cdot|s)}[f(s,a)]$
\item Bellman residual: $l_h(s,a) := \hat{Q}_h(s,a) - (\mathcal T_h\hat{Q}_{h+1})(s,a)$
\end{itemize}

Then, for all elements $x \in \mathcal{S}$, the following holds:

$$
\begin{aligned}
\hat{V}_1(x) - V_1^{\pi'}(x) =
&\sum_{h=1}^H \mathbb{E}_{\pi'} \left[(\mathbb J^{\pi}_h-\mathbb J^{\pi'}_h) \hat{Q}_h(s_h)\mid s_1 = x \right] \\
&+ \sum_{h=1}^H \mathbb{E}_{\pi'} \left[ \hat{Q}_h(s_h, a_h) - (\mathcal{T}_h \hat{Q}_{h+1})(s_h, a_h) \mid s_1 = x \right]
\end{aligned}
$$
\end{lemma}

\begin{proof}
$$
\begin{aligned}
\hat{V}_h(x) - V_h^{\pi'}(x) &= \mathbb{J}_{h}^{\pi} \hat Q_h(x) - \mathbb{J}_{h}^{\pi'} Q_h^{\pi'}(x) \\
&=\mathbb{J}^{\pi}_h\hat{Q}_h(x)-\mathbb{J}^{\pi'}_h\hat{Q}_h(x) + \mathbb{J}^{\pi'}_h\hat{Q}_h(x) -\mathbb{J}_{h}^{\pi'} Q_h^{\pi'}(x)\\
&=\mathbb{J}_{h}^{\pi'}(\hat Q_h - Q^{\pi'}_h)(x) + (\mathbb{J}_h^{\pi}-\mathbb{J}_h^{\pi'})\hat Q_h(x)\\
&= \mathbb{J}_{h}^{\pi'}(l_h + \mathcal{T_h}\hat Q_{h+1} -r_h-\mathbb{P}_hV_{h+1}^{\pi'})(x)+ (\mathbb{J}_h^{\pi}-\mathbb{J}_h^{\pi'})\hat Q_h(x)\\
&=\mathbb{J}_{h}^{\pi'}(l_h + \mathbb{P}_h\hat V_{h+1}-\mathbb{P}_hV_{h+1}^{\pi'})(x)+ (\mathbb{J}_h^{\pi}-\mathbb{J}_h^{\pi'})\hat Q_h(x)\\
&= \mathbb{J}_{h}^{\pi'}l_h(x)+ \mathbb{J}_{h}^{\pi'}\mathbb P_h(\hat V_{h+1} - V_{h+1}^{\pi'})(x) + (\mathbb{J}_h^{\pi}-\mathbb{J}_h^{\pi'})\hat Q_h(x)
\end{aligned}
$$

Using recurrence relations and the boundary condition $\hat{V}_{H+1} = V_{H+1}^{\pi'} \equiv 0$, we can derive that
\begin{equation*}
\begin{aligned}
\hat{V}_1(x) - V_1^{\pi'}(x) &= \sum_{h=1}^{H} \left(\prod_{k=1}^{h-1}\mathbb{J}_k^{\pi'}\mathbb{P}_k\right) \mathbb{J}_h^{\pi'} l_h(x) \\
&\quad + \sum_{h=1}^{H} \left(\prod_{k=1}^{h-1}\mathbb{J}_k^{\pi'}\mathbb{P}_k\right) (\mathbb{J}_h^{\pi}-\mathbb{J}_h^{\pi'}) \hat Q_h(x)
\end{aligned}
\end{equation*}
Which complete our proof.
\end{proof}

let $\pi'$ be the optimal policy $\pi^*$,$\pi$ be the greedy policy induced by $\{\hat Q_h\}_{h=1}^{H}$, and $\{\hat V_h\}_{h=1}^H$ be the corresponding value function.Then, the suboptimal bound is given by:
\begin{equation*}
    \begin{aligned}
        &V_1^*(x) - V^{\pi}_1(x) = V_1^*(x) - \hat V_1(x) + \hat V_1(x) - V^{\pi}_1(x)\\
        & = \underbrace{\sum_{h=1}^H\mathbb{E}_{\pi^*}\left[\mathcal{T}_{h}\hat{Q}_{h+1}(s_h,a_h) - \hat{Q}_{h}(s_h,a_h)\mid s_1=x\right ] }_{\textcircled{1}}+\underbrace{ \sum_{h=1}^H\mathbb{E}_{\pi}\left[ \hat{Q}_{h}(s_h,a_h)-\mathcal{T}_{h}\hat{Q}_{h+1}(s_h,a_h)\mid s_1=x\right]}_{\textcircled{2}}\\
        &+\underbrace{\sum_{h=1}^H \mathbb{E}_{\pi^*} \left[(\mathbb{J}_{h}^{\pi^*} -\mathbb{J}_{h}^{\pi})\hat{Q}_h \mid s_1 = x \right]}_{\textcircled{3}}\\
    \end{aligned}
\end{equation*}
Since $\pi$ is the greedy policy with respect to $\hat Q$, we have  $\textcircled{3} \leq 0$, and for $\textcircled{1}$ we can drive that:
\begin{equation*}
\begin{aligned}
    \textcircled{1} &\leq         \sum_{h=1}^H\mathbb{E}_{\pi^*}\left[|\mathcal{T}_{h}\hat{Q}_{h+1}(s_h,a_h) - \hat{Q}_{h}(s_h,a_h)|\big| s_1=x\right ] \\
        & \leq\sum_{h=1}^H\sqrt{\mathbb{E}_{\pi^*}\left[\left(\mathcal{T}_{h}\hat{Q}_{h+1}(s_h,a_h) - \hat{Q}_{h}(s_h,a_h)\right)^2\mid s_1=x\right ]}\\
        & \leq \sqrt{H}\sqrt{\sum_{h=1}^H \mathbb{E}_{\pi^*}\left[\left(\mathcal{T}_{h}\hat{Q}_{h+1}(s_h,a_h) - \hat{Q}_{h}(s_h,a_h)\right)^2\mid s_1=x\right ]} 
\end{aligned}
\end{equation*}
The last step makes use of the Cauchy-Schwarz inequality, and the second step employs Jensen's inequality.

Similarly, we can derive that $\textcircled{2}$ also satisfies:
\begin{equation*}
    \textcircled{2} \leq  \sqrt{H}\sqrt{\sum_{h=1}^H \mathbb{E}_{\pi}\left[\left(\mathcal{T}_{h}\hat{Q}_{h+1}(s_h,a_h) - \hat{Q}_{h}(s_h,a_h)\right)^2\mid s_1=x\right ]}
\end{equation*}

By combining the above results, we complete the proof of Theorem~\ref{thm:subopt-f}.

\subsection{Proof of theorem~\ref{thm:initial_gradient}}
In this section, we prove Theorem~\ref{thm:initial_gradient}.

\begin{proof}
\begin{equation*}
    \begin{aligned}
         \nabla \mathbb E_{\mu_{h}^k}\left[(f-\mathcal{T}_h\hat f_{h+1}^k)^2 \right] \bigg|_{\hat f_{h}^{k-1}} &=\mathbb{E}_{\mu_{h}^k}\left[ 2\left(f - \mathcal T_h\hat f_{h+1}^{k}\right)\nabla f\bigg|_{\hat f_{h}^{k-1}}\right] \\
         &=\mathbb E_{\mu_{h}^k}\left[2\left(f-\mathcal{T}_h\hat f_{h}^{k-1}+\mathcal{T}_h\hat f_{h}^{k-1}-\mathcal T_h\hat f_{h+1}^{k} \right) \nabla f\bigg|_{\hat f_{h}^{k-1}}\right]\\
         &=\mathbb E_{\mu_{h}^k}\left[2\left(\mathcal{T}_h\hat f_{h}^{k-1}-\mathcal T_h\hat f_{h+1}^{k} \right) \nabla f\bigg|_{\hat f_{h}^{k-1}}\right] + \underbrace{\mathbb{E}_{\mu_{h}^k}\left[ 2\left(f - \mathcal T_h\hat f_{h+1}^{k-1}\right)\nabla
         f\bigg|_{\hat{f}_{h}^{k-1}}\right]}_{\textcircled{1}}
    \end{aligned}
\end{equation*}
Recall the define of $\hat{f}_{h}^{k-1}$ and the Proposition~\ref{pro:rec} of $\mu_h^k$,
\begin{equation*}
\begin{aligned}
    \textcircled{1} &= \nabla \mathbb{E}_{\mu_{h}^k}\left[\left(f - \mathcal T_h\hat f_{h+1}^{k-1}\right)^2\right]\bigg|_{\hat f_h^{k-1}}\\
    & = \underbrace{\frac{k-1}{k}\nabla \mathbb{E}_{\mu_{h}^{k-1}}\left[\left(f - \mathcal T_h\hat f_{h+1}^{k-1}\right)^2\right]\bigg|_{\hat f_h^{k-1}}}_{=0} + \frac{1}{k}\nabla \mathbb{E}_{\hat d_{h}^{\pi^k}}\left[\left(f - \mathcal T_h\hat f_{h+1}^{k-1}\right)^2\right]\bigg|_{\hat f_h^{k-1}}
\end{aligned}
\end{equation*}
By combining the above results, we complete the proof of Theorem~\ref{thm:initial_gradient}
\end{proof}

\subsection{Entropy Regularized MDP}\label{App:entropy mdp}
In this section, we present the theoretical analysis and error bounds for the Entropy-Regularized Markov Decision Process (MDP). Specifically, the state value function with an entropy reward is defined as follows:
\begin{equation*}
\begin{aligned}
    &V^{\text{soft},\pi}_{h}(x) = \mathbb{E}\left[\sum_{t=h}^{H}\left(r_t(x_t,a_t)+\alpha\log \pi_t(a_t\mid x_t)\right)\bigg|x_h=x, a_t\sim\pi_t(\cdot|x_t)\right],\qquad \forall x \in \sS,h\in [H],\\
    &Q^{\text{soft},\pi}_{h}(x,a) = r_h(x,a) + \mathbb P_{h}V^{\text{soft},\pi}_{h+1}(x,a),\qquad \qquad \forall (x,a) \in \sS \times \sA,h\in [H]
\end{aligned}
\end{equation*}
with terminal condition $V^{\text{soft},\pi}_{H+1}\equiv0$.
Then the policy Bellman equations compactly read
\begin{equation*}
    \begin{aligned}
        &Q^{\text{soft},\pi}_{h}(x,a) = r_h(x,a) + \mathbb P_{h}V^{\text{soft},\pi}_{h+1}(x,a)\\
        &V_h^{\text{soft},\pi}(x) = \mathbb J^{\pi}_{h}(Q^{\text{soft},\pi}_{h}-\alpha \log \pi_h)(x), \qquad V^{\text{soft}}_{H+1}\equiv0
    \end{aligned}
\end{equation*}
For any function $g :\sS \times \sA \rightarrow \mathbb{R}$, define the soft value operator $\mathbb V^{\text{soft}}$ and the step-h soft optimality Bellman operator $\mathcal T_{h}^{\text{soft}}$ by
\begin{equation*}  
    \begin{aligned} 
        &\mathbb{V}^{\text{soft}}_g(s) :=\max_{\pi}\mathbb{E}_{a \sim \pi} \left[ g(s,a) - \alpha \log \pi(a|s) \right], \\
        &\left( \mathcal{T}^{\text{soft}}_h f \right)(s,a) := r(s,a) + \left( \mathbb{P}_h \mathbb{V}^{\text{soft}}_f \right)(s,a)
    \end{aligned}
\end{equation*}

We define the Boltzmann policy $\pi^{\text{soft}}_{f}$ induced by the function $f:\sS\times \sA\rightarrow \mathbb R$, which is given by:
\begin{equation*}\label{eq:boltzmann}
\pi^{\text{soft}}_{f} = \arg \max_{\pi} \mathbb{E}_{a\sim \pi}[f(s,a) - \alpha \log \pi(a|s)].
\end{equation*}

Similarly, we have the following lemma.
\begin{lemma}
Consider an entropy-regularized episodic MDP with horizon H. Let $\pi'=\{\pi_h'\}_{h=1}^{H}$ denote any policy, and let $\hat{Q} = \{\hat{Q}_h\}_{h=1}^H$ denote any set of estimated soft Q-functions. Let $\pi=\{\pi_h\}_{h=1}^{H}$ be the Boltzmann policy induced by $\hat{Q} = \{\hat{Q}_h\}_{h=1}^H$. For all $h \in [H]$, define:
\begin{itemize}
\item Value function: $\hat{V}_h(s) = \mathbb J_h^{\pi}(\hat{Q}_h-\alpha \log \pi_h)(s)$ where $\mathbb J_h^{\pi}f(s) = \mathbb E_{a \sim \pi_h(\cdot|s)}[f(s,a)]$
\item Bellman residual: $l_h(s,a) := \hat{Q}_h(s,a) - (\mathcal T_h^{\text{soft}}\hat{Q}_{h+1})(s,a)$
\item Entropy: $\mathcal{H}(\pi(\cdot|s)) = -\mathbb E_{a\sim\pi(\cdot|s)}[\log \pi(\cdot|s)]$
\end{itemize}
Then for all $x \in \sS$, we have
\begin{equation*}
\begin{aligned}
\hat{V}_1(x) - V_1^{\text{soft},\pi'}(x) =
&\sum_{h=1}^H \mathbb{E}_{\pi'} \left[(\mathbb J^{\pi}_h-\mathbb J^{\pi'}_h) \hat{Q}_h(s_h)+ \alpha(\mathcal H(\pi_h(\cdot|s_h))- \mathcal H(\pi'_h(\cdot|s_h))) \mid s_1 = x \right] \\
&+ \sum_{h=1}^H \mathbb{E}_{\pi'} \left[ \hat{Q}_h(s_h, a_h) - (\mathcal{T}_h^{\text{soft}} \hat{Q}_{h+1})(s_h, a_h) \mid s_1 = x \right].
\end{aligned}
\end{equation*}
\end{lemma}
\begin{proof}
    \begin{equation*}
        \begin{aligned}
            \hat{V}_h(x) - V_h^{\text{soft},\pi'}(x) &= \mathbb J^{\pi}_h(\hat{Q}_{h}-\alpha \log\pi)(x) - \mathbb J^{\pi'}_h(Q_{h}^{\text{soft},\pi'}-\alpha \log\pi')(x)\\
            &=\mathbb J^{\pi}_h\hat{Q}_{h}(x) - \mathbb J^{\pi'}_hQ_{h}^{\text{soft},\pi'}(x) + \alpha (\mathcal H(\pi_h(\cdot|x)) - \mathcal H(\pi'_h(\cdot|x)))\\
            &=\mathbb J^{\pi}_h\hat{Q}_{h}(x) - \mathbb J^{\pi'}_h\hat Q_{h}(x) +  \mathbb J^{\pi'}_h\hat Q_{h}(x) - \mathbb J^{\pi'}_hQ_{h}^{\text{soft},\pi'}(x) + \alpha (\mathcal H(\pi_h(\cdot|x)) - \mathcal H(\pi'_h(\cdot|x)))\\
            &=\mathbb J^{\pi'}_h(\hat{Q}_{h} - Q_{h}^{\text{soft},\pi'})(x)+\mathbb J^{\pi}_h\hat{Q}_{h}(x) - \mathbb J^{\pi'}_h\hat{Q}_{h}(x) + \alpha (\mathcal H(\pi_h(\cdot|x)) - \mathcal H(\pi'_h(\cdot|x)))\\
            & = \mathbb J^{\pi'}_h(l_h + \mathcal T_h^{\text{soft}}\hat{Q}_{h+1}-r-\mathbb P_{h}V^{\text{soft},\pi'}_{h+1})(x)  + \mathbb J^{\pi}_h\hat{Q}_{h}(x) - \mathbb J^{\pi'}_h\hat{Q}_{h}(x) + \alpha (\mathcal H(\pi_h(\cdot|x)) - \mathcal H(\pi'_h(\cdot|x)))\\
            & = \mathbb J_{h}^{\pi'}(l_h+\mathbb P_{h}(\hat V_{h+1}- V^{\text{soft},\pi'}_{h+1}))(x) + \mathbb J^{\pi}_h\hat{Q}_{h}(x) - \mathbb J^{\pi'}_h\hat{Q}_{h}(x) + \alpha (\mathcal H(\pi_h(\cdot|x)) - \mathcal H(\pi'_h(\cdot|x)))\\
            & = \mathbb J_{h}^{\pi'}l_h(x) + \mathbb J_{h}^{\pi'}\mathbb P_{h}(\hat V_{h+1}- V^{\text{soft},\pi'}_{h+1})(x) +\mathbb J^{\pi}_h\hat{Q}_{h}(x) - \mathbb J^{\pi'}_h\hat{Q}_{h}(x) + \alpha (\mathcal H(\pi_h(\cdot|x)) - \mathcal H(\pi'_h(\cdot|x))).
        \end{aligned}
    \end{equation*}
Using recurrence relations and the boundary condition $\hat{V}_{H+1} = V_{H+1}^{\text{soft},\pi'} \equiv 0$, we can derive that
\begin{equation*}
\begin{aligned}
\hat{V}_1(x) - V_1^{\text{soft},\pi'}(x) &= \sum_{h=1}^{H} \left(\prod_{k=1}^{h-1}\mathbb{J}_k^{\pi'}\mathbb{P}_k\right) \mathbb{J}_h^{\pi'} l_h(x) \\
&\quad + \sum_{h=1}^{H} \left(\prod_{k=1}^{h-1}\mathbb{J}_k^{\pi'}\mathbb{P}_k\right)\left( (\mathbb{J}_h^{\pi}-\mathbb{J}_h^{\pi'}) \hat Q_h(x) + \alpha \left(\mathcal H(\pi_h(\cdot|x)) - \mathcal H(\pi'_h(\cdot|x))\right) \right).
\end{aligned}
\end{equation*}
This completes the proof.
\end{proof}

\begin{theorem}[Suboptimality bound for entropy-regularized MDP via squared Bellman residuals]
Fix horizon H. Let $\{\hat Q_h\}_{h=1}^{H}$ be the soft value estimates. Define $\{\pi_{\hat Q,h}\}_{h=1}^H$ as the Boltzmann policy induced by $\{\hat Q_h\}_{h=1}^{H}$. Let $\{\pi^*_h\}_{h=1}^H$ be the optimal policy.

For functions $f,g:\sS\times \sA\rightarrow \mathbb R$, define the step-h squared Bellman residual:
\begin{equation*}
    \Delta_h (f,g)(s,a) = \left(f(s,a)-\mathcal{T}_h^{\text{soft}}g(s,a)\right)^2.
\end{equation*}
Then we have 
    \begin{align*}\label{eq:bound}
V_1^{\text{soft},\pi^*}(x) - V_1^{\text{soft},\pi_{\hat{Q}}}(x)
&\leq \sqrt{H} \left(
    \sqrt{ \mathbb{E}_{\pi^*} \!\left[ \sum_{h=1}^H \Delta_h(\hat{Q}_h, \hat{Q}_{h+1})(s_h, a_h) \,\bigg|\, s_1 = x \right] }
    + \right. \nonumber \\
&\qquad \left.
    \sqrt{ \mathbb{E}_{\pi_{\hat{Q}}} \!\left[ \sum_{h=1}^H \Delta_h(\hat{Q}_h, \hat{Q}_{h+1})(s_h, a_h) \,\bigg|\, s_1 = x \right] }
\right).
\end{align*}
\end{theorem}

\begin{proof}
    \begin{align*}
        &V_1^{\text{soft},\pi^*}(x) - V_1^{\text{soft},\pi_{\hat{Q}}}(x) = V_1^{\text{soft},\pi^*}(x) - \hat{V}_1(x) + \hat{V}_1(x) - V_1^{\text{soft},\pi_{\hat{Q}}}(x)\\
        &= \underbrace{\sum_{h=1}^H \mathbb{E}_{\pi^*}\left[\mathcal{T}_{h}^{\text{soft}}\hat{Q}_{h+1}(s_h,a_h) - \hat{Q}_{h}(s_h,a_h) \mid s_1 = x\right]}_{\textcircled{1}} + \underbrace{\sum_{h=1}^H \mathbb{E}_{\pi_{\hat{Q}}}\left[\hat{Q}_{h}(s_h,a_h) - \mathcal{T}_{h}^{\text{soft}}\hat{Q}_{h+1}(s_h,a_h) \mid s_1 = x\right]}_{\textcircled{2}}\\
        &\quad + \underbrace{\sum_{h=1}^H \mathbb{E}_{\pi^*}\left[ \left(\mathbb{J}_{h}^{\pi^*} - \mathbb{J}_{h}^{\pi_{\hat{Q}}}\right)\hat{Q}_h + \alpha \left( \mathcal{H}(\pi_h^*(\cdot|s_h)) - \mathcal{H}(\pi_{\hat{Q},h}(\cdot|s_h)) \right)  \mid s_1 = x\right]}_{\textcircled{3}}.
    \end{align*}
    Since $\pi_{\hat{Q}}$ is the Boltzmann policy induced by $\{\hat{Q}_h\}_{h=1}^H$—a property that satisfies Equation~\ref{eq:boltzmann}—we can deduce that $\textcircled{3} \leq 0$. The remainder of the proof follows the same reasoning as that of Theorem~\ref{thm:subopt-f}.
\end{proof}

\section{Related Preliminaries}\label{App:exp}
In this section, we present the detailed parameters and settings of the experiments.

\subsection{Algorithm}

\paragraph{TD3} In our paper, we utilize TD3 as a representative of deterministic policies. TD3, an Actor - Critic algorithm, is widely adopted as a baseline in various decision - making scenarios and has given rise to a multitude of variants, which have established new state - of - the - art (SOTA) results on numerous occasions. Different from the traditional policy gradient method DDPG~\citep{ddpg}, TD3 makes use of two heterogeneous critic networks, denoted as 
$Q_{\theta_{1,2}}$, to alleviate the problem of over - optimization in Q - learning. Thus, the loss function of the critics is
\begin{equation*}
\mathcal L_Q(\theta_i) = \mathbb{E}_{a,s,r,s'} \left[ (y - Q_{\theta_i}(s,a))^2 \right] \text{ for } \forall i \in \{1,2\}.
\end{equation*}
Where $y = r + \gamma \min_{j=1,2} Q_{\tilde{\theta}_j}(s', \pi_{\phi}(s'))$, $\tilde{\theta}$ denotes the target network parameters. The actor is updated according to the Deterministic Policy Gradient:
\begin{equation*}
\nabla_{\phi} J(\phi) = \mathbb{E}_s \left[ \nabla_{a} Q_{\theta_1}(s, \pi_{\phi}(s)) \nabla_{\phi} \pi_{\phi}(s) \right].
\end{equation*}

\paragraph{SAC} We select SAC as a representative of stochastic policies and combine it with \texttt{SWD} in the main experiment. SAC is devised to maximize expected cumulative rewards while also boosting exploration via the maximum entropy principle. The actor strives to learn a stochastic policy that outputs a distribution over actions, where the critics estimate the value of taking a specific action in a given state. This enables a more diverse range of actions, facilitating better exploration of the action space. In traditional reinforcement learning, the objective is to maximize the expected return. However, SAC introduces an additional term that maximizes the entropy of the policy, encouraging exploration. The objective function for optimizing the policy is given by:
\begin{equation*}
J(\pi) = \mathbb{E}_{s_t, a_t} \left[ r(s_t, a_t) + \alpha \mathcal H(\pi(\cdot \vert s_t)) \right]
\end{equation*}
where \( H(\pi(\cdot \vert s_t)) \) denotes the entropy of the policy, and \( \alpha \) is a temperature parameter that balances the trade-off between the immediate reward and the policy entropy. The training procedure of SAC involves two main updates: updating the value function and updating the policy.
The value function is updated by minimizing the following loss:
\begin{equation*}
\mathcal L(Q) = \mathbb{E}_{(s,a,r,s') \sim D} \left[ \frac{1}{2} \left( Q(s,a) - \left( r + \gamma V(s') \right) \right)^2 \right]
\end{equation*}
where \( \gamma \) is the discount factor, dictating the weight assigned to future rewards. \( V(s') \) denotes the value function of the next state, which is typically approximated using a separate neural network. The policy is updated by maximizing the following objective:
\begin{equation*}
J(\pi) = \mathbb{E}_{s_t \sim D} \left[ \mathbb{E}_{a_t \sim \pi} \left[ Q(s_t, a_t) - \alpha \log \pi(a_t \vert s_t) \right] \right]
\end{equation*}
Here, \( -\alpha \log \pi(a_t \vert s_t) \) represents the entropy of the policy, which serves to promote exploration.

\paragraph{SimBa} We adopt SimBa~\citep{lee2025simba} as our SAC network architecture, which is specifically designed for reinforcement learning (RL) scenarios. Distinctive for embedding a "simplicity bias," SimBa not only mitigates overfitting but also enables parameter scaling in deep RL—addressing two key challenges in large-scale RL model training. Concretely, SimBa comprises three core components: (i) an observation normalization layer that standardizes input data using running statistics, ensuring stable data distribution for subsequent layers; (ii) a residual feedforward block that establishes a direct linear pathway from input to output, facilitating gradient propagation and preserving low-complexity feature representations; and (iii) a layer normalization module that regulates feature magnitudes, preventing excessive value drift during training.

\paragraph{Prioritized Experience Replay}
We adopt Prioritized Experience Replay (PER)~\citep{PER} to bias sampling toward transitions that are expected to yield larger learning progress. Instead of drawing mini-batches uniformly from the replay buffer, PER assigns each transition \(i\) a priority \(p_i\) based on its temporal-difference (TD) error and samples proportionally:
\[
\delta_i = \bigl| r_i + \gamma\, \hat{V}(s'_i) - Q(s_i,a_i) \bigr|, 
\qquad
p_i = \bigl(\delta_i + \varepsilon\bigr)^{\alpha},
\qquad
P(i) = \frac{p_i}{\sum_j p_j},
\]
where \(\varepsilon > 0\) avoids zero priorities, \(\alpha \in [0,1]\) controls the degree of prioritization (\(\alpha=0\) recovers uniform sampling). To correct the sampling bias introduced by \(P(i)\), PER uses importance-sampling (IS) weights
\[
w_i \;=\; \Bigl(\tfrac{1}{N\,P(i)}\Bigr)^{\beta}, 
\qquad
\tilde{w}_i \;=\; \frac{w_i}{\max_j w_j},
\]
where \(N\) is the buffer size and \(\beta \in [0,1]\) is annealed toward 1 during training.

\paragraph{Gradient Magnitude-based neuron activity assessment}
We employ \texttt{GraMa}~\citep{grama}—a gradient-magnitude-driven, architecture-agnostic metric— as our plasticity metric. Specifically, for each individual neuron (or predefined parameter group), \texttt{GraMa} calculates the magnitude of gradients computed over mini-batches and maintains a normalized score for each layer; crucially, higher scores correspond to greater neural plasticity.  

Given an input distribution D, let $|\nabla h_{\ell}^{i} L(x)|$ denote the gradient magnitude of neuron i in layer $\ell$ under an input x $\in$ D, and let $H_{\ell}$ represent the number of neurons in layer $\ell$. The learning capacity score for each individual neuron by leveraging the normalized average of its corresponding layer $\ell$, as formulated below:

$$G_{\ell}^{i} = \frac{\mathbb{E}_{x \in D} \left[ |\nabla h_{\ell}^{i} L(x)| \right]}{\frac{1}{H_{\ell}} \sum_{k \in \mathcal{H}_{\ell}} \mathbb{E}_{x \in D} \left[ |\nabla h_{\ell}^{k} L(x)| \right]} $$

GraMa (Gradient Magnitude-based neuron activity assessment) identifies neuron i in layer $\ell$ as inactive if $G_{\ell}^{i} \leq \tau$, where $\tau$ denotes the predefined inactivity threshold.

\myadded{\paragraph{Double DQN} We adopt Double Deep Q-Network (DDQN)~\citep{DDQN} as our reinforcement learning (RL) baseline, specifically chosen for both pixel-based input scenarios and tasks with long time horizons. By decoupling action selection from target value estimation, DDQN effectively mitigates the overestimation bias inherent in standard DQN, ensuring more stable and accurate value learning. Concretely, while standard DQN maximizes the estimated value using the same network, DDQN utilizes the online network with parameters $\theta$ to select the optimal action and the target network with parameters $\theta^{-}$ to evaluate that action. }

\myadded{
\subsection{Relationship and Complementarity with Existing Work}\label{app:relationship}
}
\myadded{
Prior research on plasticity loss has predominantly centered on \textbf{NTK-based methods}, which we classify into three core categories based on their underlying mechanisms:}

\myadded{
(1) \textbf{Reset-based methods (leveraging random initialization properties)}:  
These approaches capitalize on a key characteristic of over-parameterized neural networks: randomly initialized networks exhibit full-rank Neural Tangent Kernel (NTK) matrices. To mitgate plasticity loss, they periodically reset network parameters to refresh the NTK and restore the model’s capacity for learning. Representative examples include:
\begin{itemize}
    \item \textbf{ReDo}~\citep{sokar2023dormant}: Employs activation-driven reinitialization to reset critical network components
    \item \textbf{ReGraMa}~\citep{grama}: Utilizes gradient information to guide parameter reinitialization, targeting degraded NTK structures
    \item \textbf{S\&P}~\citep{s&p}: Introduces controlled noise into network parameters to reactivate dormant plasticity
    \item \textbf{Plasticity Injection}~\citep{plasticityinjection}: Under the premise of keeping the output unchanged, thoroughly refresh the final linear layer.
\end{itemize}
}
\myadded{
(2) \textbf{Implicit NTK regularization methods}:  
This category focuses on detecting early signs of NTK rank deficiency—such as unconstrained parameter norm growth—and implementing targeted constraints to avert rank collapse. Key strategies within this framework are:
\begin{itemize}
    \item \textbf{Reducing Churn}~\citep{Tang2024ReduceChurn}: Suppresses off-diagonal elements of the NTK matrix to minimize gradient correlations, while dynamically adjusting step sizes in reinforcement learning (RL) settings to preserve NTK integrity
    \item \textbf{Auxiliary-loss-based representation stabilization}~\citep{moalla2024no}: Integrates additional loss terms to stabilize feature representations, indirectly safeguarding NTK rank
\end{itemize}
}
\myadded{
(3) \textbf{Architecture-based methods}:  
These approaches address plasticity loss at the network design level, either by constructing inherently larger and more robust architectures or by dynamically expanding parameter counts during training to prevent NTK rank collapse. Notable instances include:
\begin{itemize}
    \item \textbf{Hyperspherical Normalization for Scalable Deep RL}~\citep{simbav2}: Designs architectures with built-in stability, leveraging hyperspherical normalization to maintain NTK full-rank properties
    \item \textbf{Forget-and-Grow Strategy for Deep RL Scaling}~\citep{fog}: Implements dynamic parameter expansion to sustain NTK rank and preserve plasticity
\end{itemize}}

\myadded{
Our work is \textbf{fundamentally orthogonal} to these NTK-based paradigms. Unlike existing methods— which tackle plasticity loss through architectural modifications, explicit NTK regularization, or parameter resetting—we adopt a \textbf{novel gradient dynamics perspective}: our core objective is to mitigate the \textbf{temporal distribution shift in the replay buffer}, a primary driver of gradient magnitude decay and subsequent plasticity loss. Theoretically, our distribution-aware sampling strategy does not overlap with NTK-based plasticity preservation techniques; instead, it offers a complementary approach to addressing the root causes of plasticity loss in deep learning systems.}

\myadded{
\section{Approximate bucket-based sampling}\label{app:approix}
\paragraph{Efficient Approximation via Bucket Sampling}
To mitigate the computational overhead of recalculating weights for the entire replay buffer, we exploit the \textbf{monotonic age property} of our weighting scheme. Since the weights are strictly determined by the temporal age of transitions, we propose a bucket-based approximation method:
\begin{enumerate}
    \item \textbf{Partitioning:} We divide the $N$ transitions in the buffer into $B$ sequential buckets (where $B \ll N$).
    \item \textbf{Approximation:} Leveraging the monotonicity, we estimate the total weight of each bucket using the weight of its median sample, significantly reducing calculation redundancy.
    \item \textbf{Hierarchical Sampling:} We first sample a bucket according to the approximated probability distribution, then uniformly sample a transition within that bucket.
\end{enumerate}
As shown in Table~\ref{tab:complexity_comparison}, this approach reduces the sampling complexity from $\mathcal{O}(N)$ to $\mathcal{O}(B)$. With $B=2000$ and a buffer size of $N=10^6$, this yields a theoretical \textbf{500$\times$ speedup} in the weight computation phase, rendering the overhead negligible.}

\begin{table}[htbp]
    \centering
    \caption{\myadded{Computational complexity comparison. $N$ denotes the buffer size ($10^6$), $M$ the batch size, and $B$ the number of buckets ($2000$).}}
    \label{tab:complexity_comparison}
    \begin{tabular}{lcc}
        \toprule
        \textbf{Method} & \textbf{Complexity} & \textbf{Scale Dependency} \\
        \midrule
        Uniform Sampling & $\mathcal{O}(M)$ & Independent of Buffer \\
        Exact SWD & $\mathcal{O}(N + M)$ & Linear w.r.t Buffer \\
        \textbf{Approximate SWD} & $\mathbf{\mathcal{O}(B + M)}$ & \textbf{Linear w.r.t Buckets} \\
        \bottomrule
    \end{tabular}
\end{table}

\myadded{
\paragraph{Empirical Validation} We validate the efficiency and effectiveness of this approximation on the \texttt{Humanoid-run} task. As presented in Table~\ref{tab:runtime_perf}, the Approximate SWD method matches the wall-clock training time of Uniform sampling (approx. 8.7 hours) while preserving the performance gains of the exact method, achieving a high episode return of $224.9 \pm 17.5$.
}
\begin{table}[htbp]
    \centering
    \caption{\myadded{Runtime and performance comparison on \texttt{Humanoid Run}. The approximate method retains performance while significantly reducing training time.}}
    \label{tab:runtime_perf}
    \begin{tabular}{lcc}
        \toprule
        \textbf{Method} & \textbf{Wall-Clock Time} & \textbf{Episode Return} \\
        \midrule
        Uniform & 8.65 h & $190.46 \pm 7.99$ \\
        Exact SWD & 10.43 h & $229.01 \pm 37.43$ \\
        \textbf{Approximate SWD} & \textbf{8.70 h} & \textbf{224.93 $\pm$ 17.47} \\
        \bottomrule
    \end{tabular}
\end{table}

\section{Experimental Details}
\subsection{Structure}
\paragraph{TD3}
In this paper, we adopt the official network architecture of Twin Delayed Deep Deterministic Policy Gradient (TD3) for baseline comparison, with detailed layer-wise configurations provided in Table~\ref{tab:td3_network}.

\begin{table}[htbp]
  \centering
  \caption{Network Structures of the Twin Delayed Deep Deterministic Policy Gradient (TD3)}
  \label{tab:td3_network}  
  \begin{tabular}{lcc}
    \toprule
    \textbf{Network Component}       & \textbf{Actor Network}               & \textbf{Critic Network$^\dagger$}    \\
    \midrule
    Fully Connected Layer        & (state\_dim) $\to$ (256)             & (state\_dim + action\_dim) $\to$ (256) \\
    Activation        & ReLU                                  & ReLU                                  \\
    Fully Connected Layer        & (256) $\to$ (128)                     & (256) $\to$ (128)                     \\
    Activation            & ReLU                                  & ReLU                                  \\
    Output Fully Connected Layer     & (128) $\to$ (action\_dim)             & (128) $\to$ (1)                       \\
    Activation                & Tanh$^\ddagger$                       & None                                  \\
    \bottomrule
  \end{tabular}
  \vspace{0.5em}  
  \footnotesize{
  \\
    $^\dagger$: TD3 adopts two identical critic networks (Critic 1 \& Critic 2) for delayed Q-value update, both following the above structure; \\
    $^\ddagger$: Tanh activation constrains the actor's output action to the range $[-1, 1]$, consistent with standard continuous action space settings.
  }
\end{table}

\myadded{\paragraph{Double DQN}
We adopt the Nature CNN network architecture, the detailed specifications of which are presented in Table~\ref{tab:dqn_encoder} and Table~\ref{tab:dqn_network}. Our implementation refers to the official code repository\footnote{\url{https://github.com/google-deepmind/dqn}} to ensure consistency with the original design. }
\begin{table}[ht]
    \centering
    \caption{\myadded{Architecture of the Nature CNN Encoder used in Double DQN. The input consists of 4 stacked frames of size $84 \times 84$.}}
    \label{tab:dqn_encoder}
    \begin{tabular}{lcccc}
        \toprule
        \textbf{Layer} & \textbf{Input Channels} & \textbf{Kernel Size / Stride} & \textbf{Output Channels} & \textbf{Activation} \\
        \midrule
        Conv1 & 4  & $8 \times 8$ / 4 & 32 & ReLU \\
        Conv2 & 32 & $4 \times 4$ / 2 & 64 & ReLU \\
        Conv3 & 64 & $3 \times 3$ / 1 & 64 & ReLU \\
        \bottomrule
    \end{tabular}
\end{table}

\begin{table}[ht]
    \centering
    \caption{Architecture of the Double DQN Q-Network. The input is the flattened feature vector from the Encoder.}
    \label{tab:dqn_network}
    \begin{tabular}{lcc}
        \toprule
        \textbf{Layer} & \textbf{Configuration} & \textbf{Activation} \\
        \midrule
        Input (Flatten) & $3136$ units ($7 \times 7 \times 64$) & - \\
        FC1 & Linear($3136 \to 512$) & ReLU \\
        Output & Linear($512 \to |\mathcal{A}|$) & - \\
        \bottomrule
    \end{tabular}
\end{table}

\paragraph{SAC}
In this paper, we adopt the same configuration of SimBa as used in the Soft Actor-Critic (SAC) algorithm, with detailed network structures provided in Table~\ref{tab:simba_resblock}, Table~\ref{tab:simba_encoder}, and Table~\ref{tab:simba_sac_network}. Our implementation refers to the official SimBa code repository\footnote{\url{https://github.com/SonyResearch/simba}} to ensure consistency with the original design.

\begin{table}[!htbp]
  \centering
  \caption{Architecture of the SimBa Residual Block}
  \label{tab:simba_resblock}  
  \begin{tabular}{lcc}
    \toprule
    \textbf{Layer/Operation}       & \textbf{Input/Output Dimensions} & \textbf{Activation Function} \\
    \midrule
    Layer Normalization            & (hidden\_dim) $\to$ (hidden\_dim) & None \\
    Fully Connected (Expansion)    & (hidden\_dim) $\to$ (4$\times$hidden\_dim) & ReLU \\
    Fully Connected (Compression)  & (4$\times$hidden\_dim) $\to$ (hidden\_dim) & None \\
    Residual Connection            & Input $\oplus$ Block Output$^*$ & None \\
    \bottomrule
  \end{tabular}
  \vspace{0.5em}  
  \footnotesize{
  \\$^*$: "$\oplus$" denotes element-wise addition between the original input and the block output.}
\end{table}

\begin{table}[!htbp]
  \centering
  \caption{Architecture of the SimBa Encoder}
  \label{tab:simba_encoder}
  \begin{tabular}{lc}
    \toprule
    \textbf{Component}             & \textbf{Structure \& Dimension Flow} \\
    \midrule
    Input Projection (Fully Connected) & (input\_dim) $\to$ (hidden\_dim) \\
    Residual Block Stack            & $\times$ num\_blocks$^\dagger$ (each block follows Table~\ref{tab:simba_resblock}) \\
    Final Layer Normalization       & (hidden\_dim) $\to$ (hidden\_dim) \\
    \bottomrule
  \end{tabular}
  \vspace{0.5em}
  \footnotesize{\\$^\dagger$: "num\_blocks" denotes the number of stacked residual blocks, configurable based on task requirements.}
\end{table}

\begin{table}[!htbp]
  \centering
  \caption{Network Structures of the SimBa-SAC Framework}
  \label{tab:simba_sac_network}  
  \begin{tabular}{lcc}
    \toprule
    \textbf{Component}             & \textbf{Actor Network} & \textbf{Critic Network} \\
    \midrule
    Input Dimension                & (state\_dim) & (state\_dim + action\_dim) \\
    SimBa Encoder    & hidden\_dim=128; num\_blocks=1 & hidden\_dim=512; num\_blocks=2 \\
    Fully Connected & (128) $\to$ (action\_dim) & (512) $\to$ (1) \\
    Output Activation              & Tanh$^\ddagger$ & None \\
    \bottomrule
  \end{tabular}
  \vspace{0.5em}
  \footnotesize{\\$^\ddagger$: Tanh activation is used to constrain the action output within the range $[-1, 1]$, consistent with standard SAC implementations.}
\end{table}

\subsection{Implementation Details}
Our codes are implemented with Python 3.10 and JAX. All experiments were run on NVIDIA GeForce GTX 3090 GPUs. Each single training trial ranges from 10 hours to 21 hours, depending on the algorithms and environments.

\paragraph{TD3 Implementation}
Our TD3 implementation refers to CleanRL\footnote{\url{https://github.com/vwxyzjn/cleanrl/blob/master/cleanrl/td3_continuous_action.py}}, an efficient and reliable repository for reinforcement learning (RL) algorithm implementations. 

Notably, for all OpenAI MuJoCo experiments, we directly use the raw state and reward signals from the environment without any normalization or scaling. To facilitate exploration, an exploration noise sampled from $\mathcal{N}(0, 0.1)$ is added to the action selection process of all baseline methods. The discount factor is set to 0.99, and the Adam optimizer is adopted for all algorithms.

Table~\ref{tab:td3_hparams} presents the complete hyperparameters of TD3 used in our experiments; to reproduce the learning curves reported in the main text, we recommend using random seeds 1 to 5.

\begin{table*}[!htbp]
  \centering
  \caption{Hyperparameters of the TD3 Algorithm} 
  \label{tab:td3_hparams} 
  \begin{tabular}{lc}
    \toprule
    \textbf{Hyperparameter}               & \textbf{TD3 Configuration} \\
    \midrule
    Actor Learning Rate                   & $10^{-4}$                  \\
    Critic Learning Rate                  & $10^{-3}$                  \\
    Discount Factor                       & $0.99$                     \\
    Batch Size                            & $128$                      \\
    Replay Buffer Size                    & $10^6$                     \\
    \midrule
    \multicolumn{2}{l}{\textbf{\texttt{SWD}-Specific Hyperparameters}} \\  
    Linear Decay Steps                    & $100,000$                  \\
    Minimum Weight ($\text{min\_weight}$) & $0.1$                      \\
    \bottomrule
  \end{tabular}
\end{table*}

\myadded{\paragraph{Double DQN Implementation}
Our Double DQN implementation builds upon the CleanRL repository\footnote{\url{https://github.com/vwxyzjn/cleanrl/blob/master/cleanrl/dqn_atari.py}}, recognized for its high-fidelity and reproducible reference algorithms. To ensure experimental fairness, we strictly align our configuration with standard Atari benchmarks.}

\begin{table}[!htbp]
  \centering
  \caption{\myadded{Hyperparameters of Our Double DQN Implementation}}
  \label{tab:ddqn_hparams}  
  \begin{tabular}{lc}
    \toprule
    \textbf{Hyperparameter}                  & \textbf{Value} \\
    \midrule
    \multicolumn{2}{l}{\textit{General Training}} \\
    Optimizer                                & Adam \\
    Learning Rate                            & $1 \times 10^{-4}$ \\
    Discount Factor ($\gamma$)               & $0.99$ \\
    Buffer Size                              & $1 \times 10^6$ \\
    Batch Size                               & $32$ \\
    Learning Starts                          & $80,000$ steps \\
    Train Frequency                          & $4$ steps \\
    Total Timesteps                          & $10 \text{ M}$ \\
    \midrule
    \multicolumn{2}{l}{\textit{Exploration (Epsilon-Greedy)}} \\
    Start Epsilon ($\varepsilon_{\text{start}}$) & $1.0$ \\
    End Epsilon ($\varepsilon_{\text{end}}$)     & $0.01$ \\
    Exploration Fraction                     & $0.10$ ($1 \times 10^6$ steps) \\
    \midrule
    \multicolumn{2}{l}{\textit{Target Network}} \\
    Target Update Frequency                  & $1000$ steps \\
    Target Update Rate ($\tau$)              & $1.0$ (Hard Update) \\
    \midrule
    \multicolumn{2}{l}{\textbf{\texttt{SWD}-Specific Hyperparameters}} \\
    Linear Decay Steps                       & $80,000$ \\
    Minimum Weight                           & $0.1$ \\
    Number of Buckets                        & $2000$ \\
    \bottomrule
  \end{tabular}
\end{table}

\myadded{The detailed hyperparameters are presented in Table~\ref{tab:ddqn_hparams}. Notably, for the Arcade Learning Environment (ALE) tasks, we incorporate the bucket-based approximate sampling mechanism from \methodName to enhance efficiency.}

\paragraph{SAC Implementation} 
Our Soft Actor-Critic (SAC) implementation is also based on the CleanRL repository, specifically referencing the continuous action SAC implementation\footnote{\url{https://github.com/vwxyzjn/cleanrl/blob/master/cleanrl/sac_continuous_action.py}}.

\begin{table*}[!htbp]
  \centering
  \caption{Hyperparameters of Our SAC Implementation (with SimBa Encoder)}
  \label{tab:sac_hparams}  
  \begin{tabular}{lc}
    \toprule
    \textbf{Hyperparameter}                  & \textbf{SAC (with SimBa Encoder)} \\
    \midrule
    Optimizer                                & AdamW (weight decay = $10^{-2}$)        \\
    Policy (Actor) Learning Rate             & $1 \times 10^{-4}$                      \\
    Q-Network (Critic) Learning Rate         & $1 \times 10^{-4}$                      \\
    Discount Factor                          & $0.99$                                  \\
    Batch Size                               & $256$                                   \\
    Warmup Steps (for Policy Update)         & $5000$                                  \\
    Target Q-Network Update Rate ($\tau$)    & $0.005$                                 \\
    Target Q-Network Update Interval         & $1$ (step)                              \\
    Policy (Actor) Update Interval           & $2$ (steps, policy\_frequency)          \\
    Entropy Target                            & $-|A|$ ($|A|$ = action space dimension)  \\
    SimBa Encoder (Actor): Hidden Dim / Blocks & $128$ / $1$                             \\
    SimBa Encoder (Critic): Hidden Dim / Blocks & $512$ / $2$                             \\
    \midrule
    \multicolumn{2}{l}{\textbf{\texttt{SWD}-Specific Hyperparameters}} \\  
    Linear Decay Steps                        & $80,000$                                \\
    Minimum Weight ($\text{min\_weight}$)     & $0.1$                                   \\
    \midrule
    \multicolumn{2}{l}{\textbf{\texttt{PER}-Specific Hyperparameters}} \\
    Prioritization Exponent ($\alpha$)        & $0.6$                                   \\
    Importance Sampling Exponent ($\beta$)    & $0.4$                                   \\
    Beta Increment Rate                        & $1 \times 10^{-6}$                     \\
    \bottomrule
  \end{tabular}
\end{table*}

\noindent The hyperparameters for our SAC (equipped with the SimBa encoder) are detailed in Table~\ref{tab:sac_hparams}.

\subsection{SAC leanring curve}

\begin{figure}[H]
\centering  
\subfigure[Dog Run]{
\label{Fig: app dog run}
\includegraphics[width=0.4\textwidth]{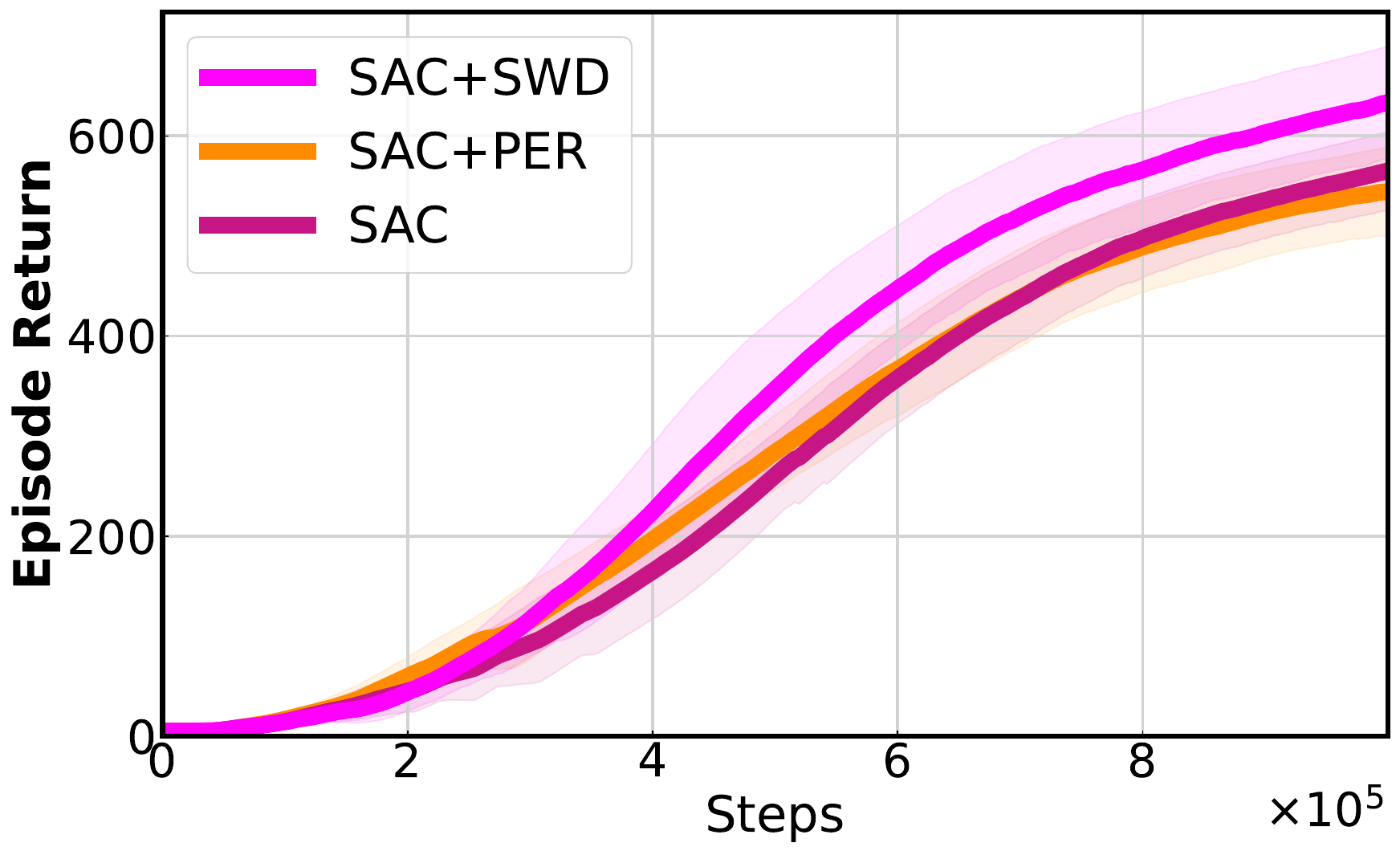}
}
\subfigure[Dog Walk]{
\label{Fig: app }
\includegraphics[width=0.4\textwidth]{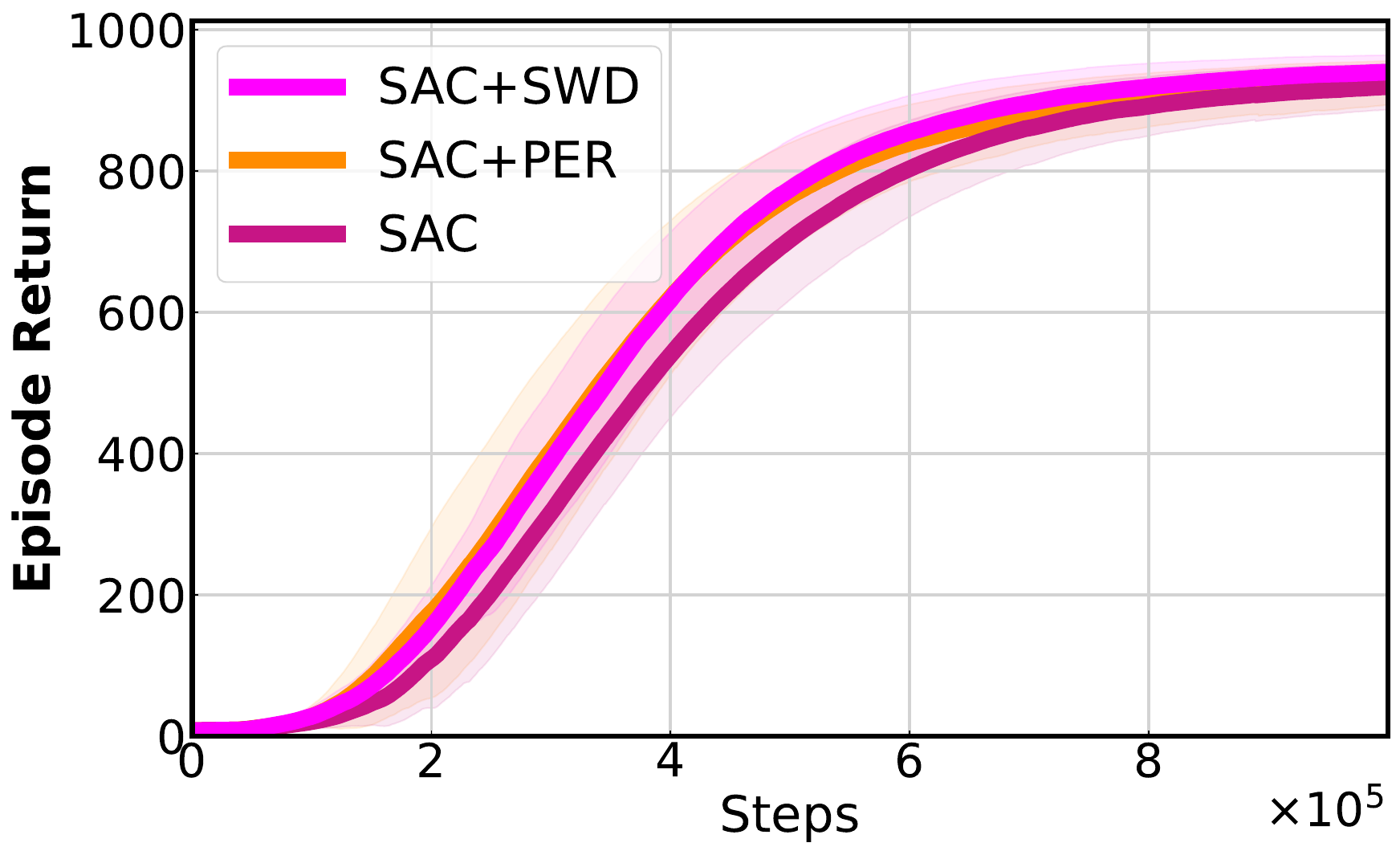}}
\subfigure[Humanoid Run]{
\label{Fig:Walker2d}
\includegraphics[width=0.4\textwidth]{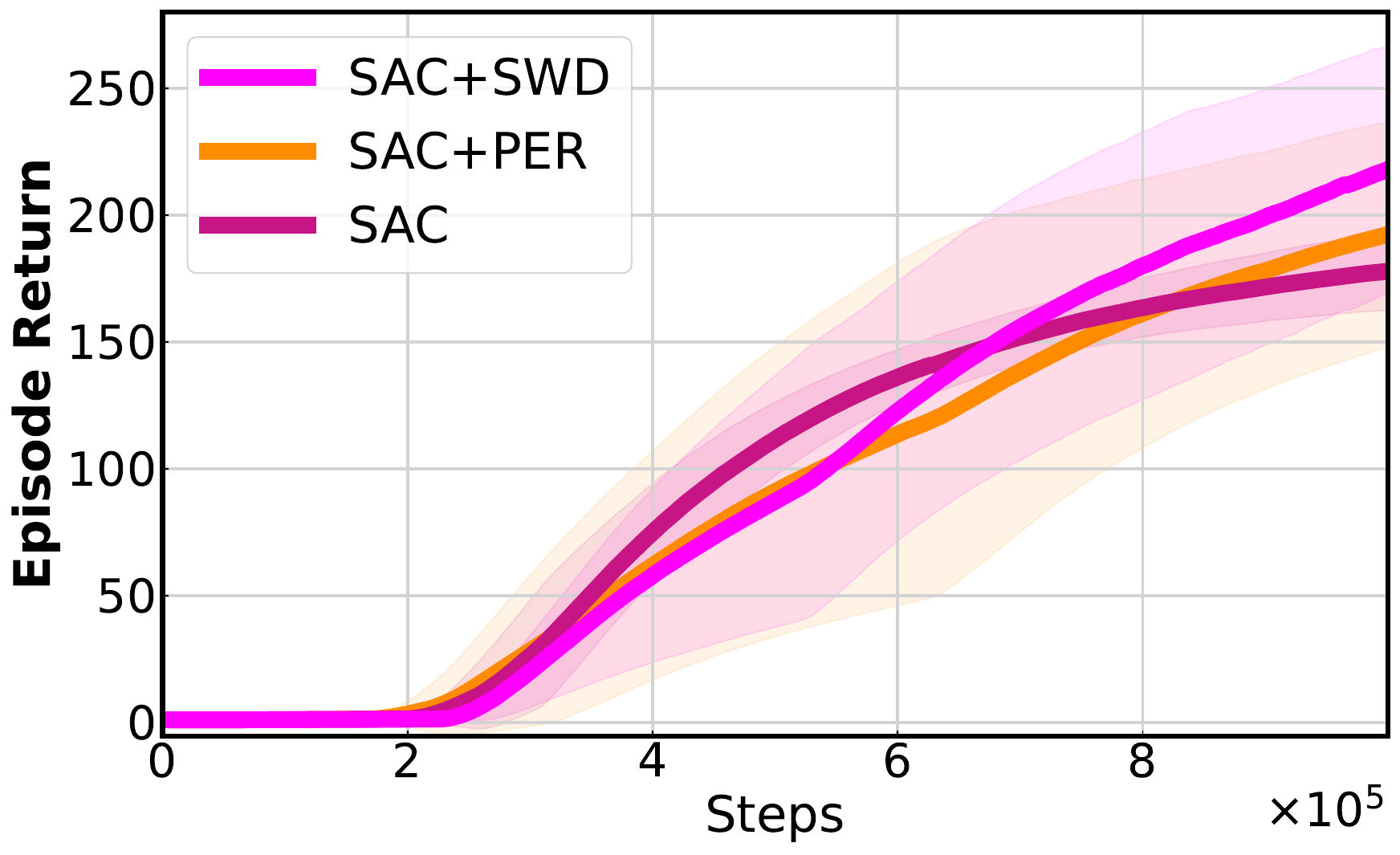}
}\subfigure[Humanoid Walk]{
\label{Fig:Walker2d}
\includegraphics[width=0.4\textwidth]{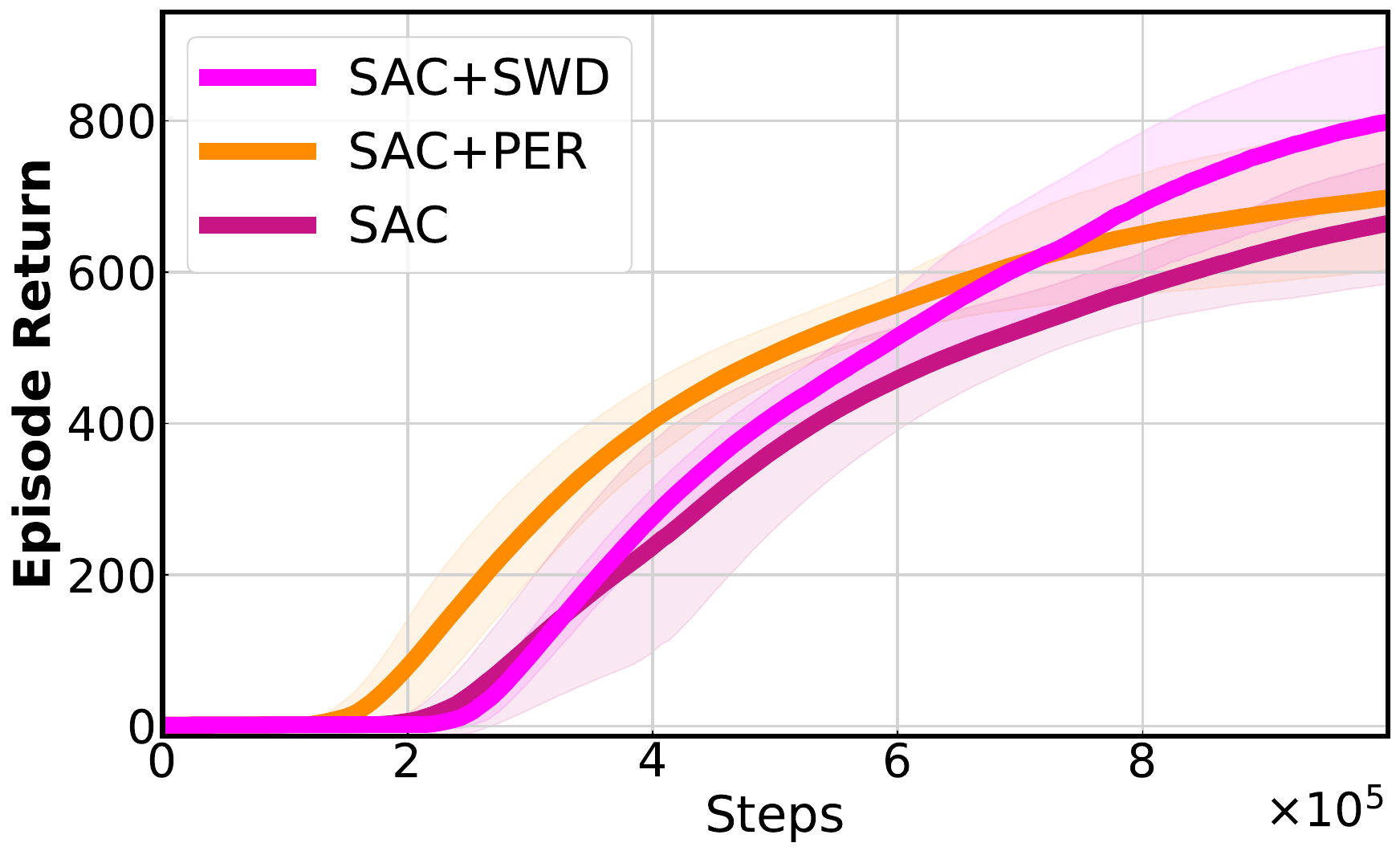}
}
\caption{SAC leanrning curve on DMC tasks}
\end{figure}

\section{Additional Experiments}\label{App:additional exp}
\subsection{\texttt{SWA}}
In this section, we provide detailed information about our ablation experiments. First, we present the algorithmic details of \texttt{SWA}, which are summarized in Algorithm~\ref{alg:linear_decay_augment}.

We adopt the detailed parameter settings of Soft Actor-Critic (SAC), as presented in Table~\ref{tab:sac_hparams}—specifically, we use the same Linear decay steps $T$
 and minimum weight $w_{\min}$ as specified therein.

\begin{algorithm}[H]
\caption{\texttt{SWA}}
\label{alg:linear_decay_augment}
\begin{algorithmic}[1]
\REQUIRE Linear decay steps $T$, minimum weight $w_{\min}$, Current time $t$, timestamps $\{t_i\}_{i=1}^{|\mathcal{D}|}$
    \FOR{$i = 1$ to $|\mathcal{D}|$}
        \STATE $age_i = t - t_i$ 
        \STATE $w_i = \min\left(1, w_{\min} + \frac{age_i}{T}\right)$ 
    \ENDFOR
    
    \STATE $p_i = \frac{w_i}{\sum_{j=1}^{|\mathcal{D}|} w_j}$ for $i = 1, \ldots, |\mathcal{D}|$ 
    \STATE $\mathcal{I} \sim \text{Categorical}(\{p_i\}_{i=1}^{|\mathcal{D}|}, B)$ 
\STATE \textbf{return} $\mathcal{B} = \{(s_i, a_i, r_i, s'_i, d_i)\}_{i \in \mathcal{I}}$
\end{algorithmic}
\end{algorithm}

\subsection{Ablation Study of Update-to-Data}

\begin{figure}[H]
\centering  
\subfigure[UTD=1]{
\label{Fig:ant}
\includegraphics[width=0.32\textwidth]{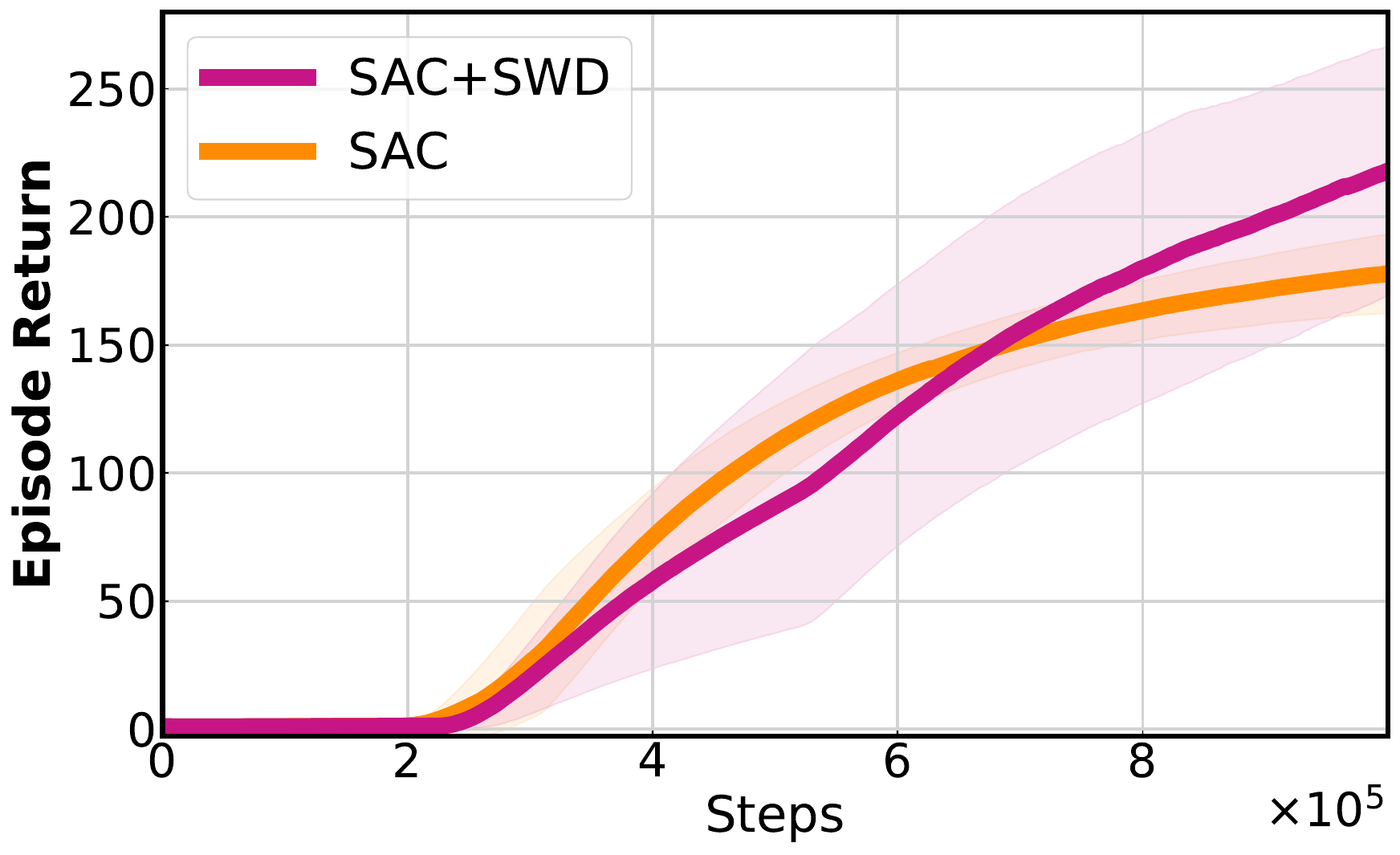}
}
\subfigure[UTD=2]{
\label{Fig:Halfcheef}
\includegraphics[width=0.32\textwidth]{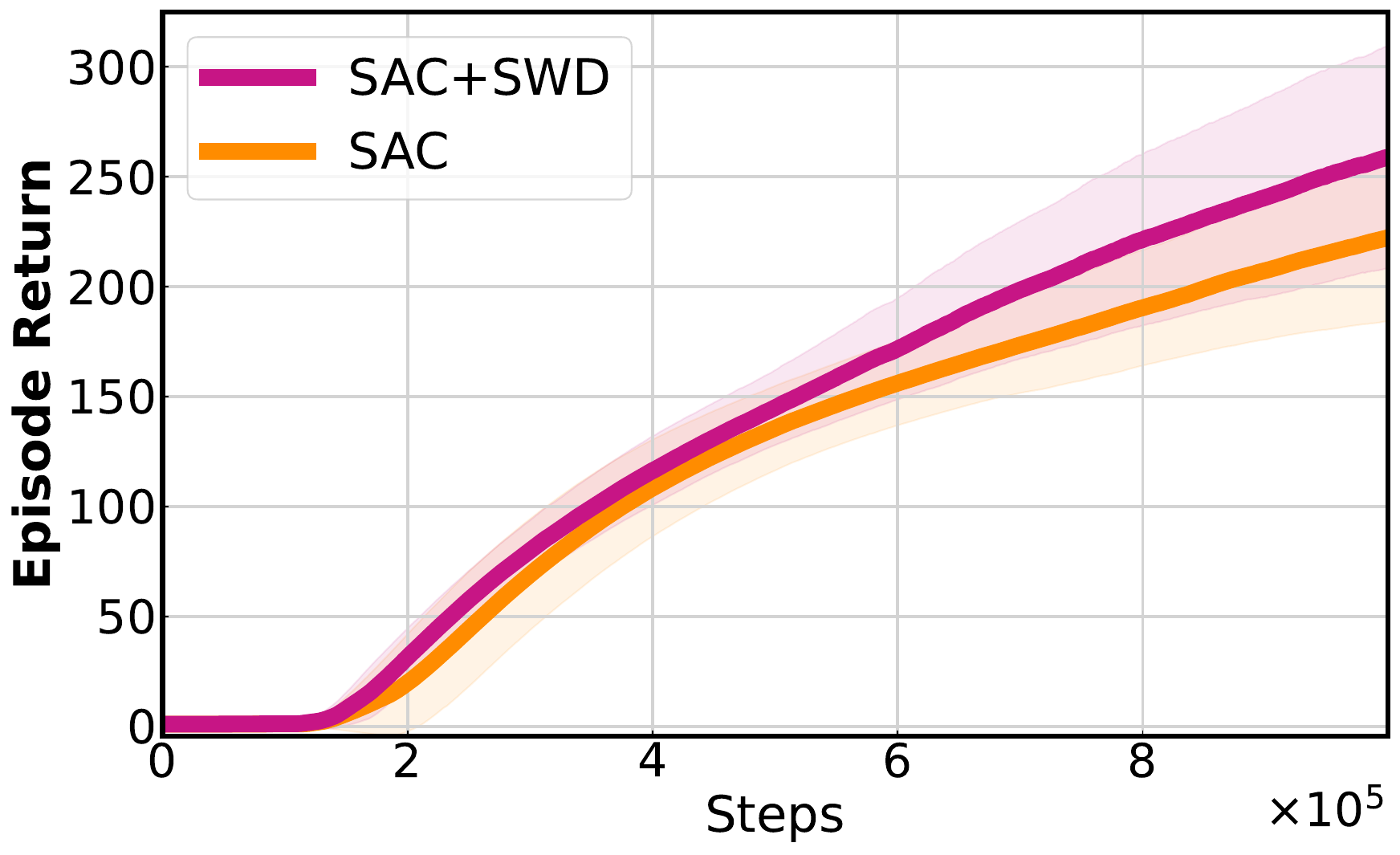}}
\subfigure[UTD=5]{
\label{Fig:Walker2d}
\includegraphics[width=0.32\textwidth]{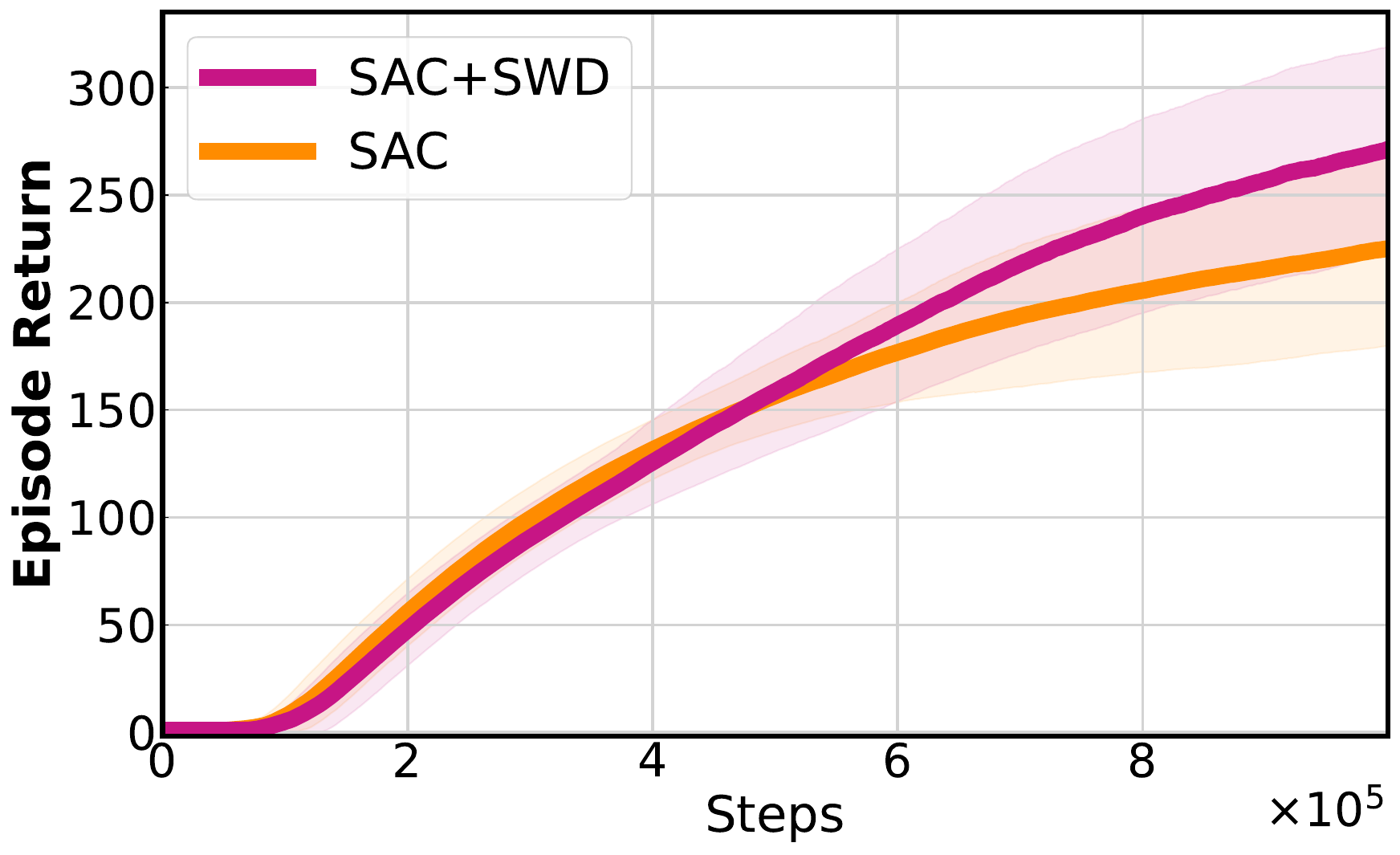}
}
\caption{Sensitivity analysis regarding the UTD. Data represents the mean $\pm$ std  of five experimental runs conducted on the Humanoid Run.}
\label{Fig:utd}
\end{figure}
We adopt SAC (Soft Actor-Critic) as the backbone algorithm and aim to optimize the Update-to-Data (UTD) ratio. This optimization enables faster policy iteration, thereby better leveraging the advantages of \texttt{SWD}. As illustrated in Figure~\ref{Fig:utd}, with the increase in the UTD ratio, \texttt{SWD} consistently outperforms the uniform sampling baseline.

\myadded{
\subsection{Parameter Sensitivity Analysis}
To assess the robustness of our proposed method, we conducted an extensive grid search to evaluate the sensitivity of SWD to its two primary hyperparameters: the linear decay steps ($T$) and the minimum weight threshold ($w_{\min}$). }

\myadded{
We constructed a $5 \times 5$ hyperparameter grid, varying $T_{\text{decay}}$ from $20,000$ to $100,000$ and $w_{\min}$ from $0.02$ to $0.10$. Experiments were performed on the \texttt{Humanoid Run} task, with each of the 25 configurations averaged over 5 random seeds (totaling 125 independent runs).}
\myadded{
The results, summarized in Table~\ref{tab:sensitivity_clean}, indicate that SWD maintains stable performance across a wide range of hyperparameter settings. While optimal performance fluctuates slightly, the method does not exhibit drastic failure modes within the tested range, demonstrating its robustness to hyperparameter selection.}
\begin{table}[htbp]
  \centering
  \caption{\myadded{\textbf{Parameter Sensitivity Analysis.} Grid search results on \texttt{Humanoid Run} (Mean $\pm$ Std). The best performance is marked in \textbf{bold}.}}
  \label{tab:sensitivity_clean}
  \setlength{\tabcolsep}{5pt} 
  \renewcommand{\arraystretch}{1.25} 
  
  \begin{tabular}{lccccc}
    \toprule
    \textbf{Decay Steps} & \multicolumn{5}{c}{\textbf{Minimum Weight Threshold ($w_{\min}$)}} \\
    \cmidrule(lr){2-6}
    ($T_{\text{decay}}$) & \textbf{0.02} & \textbf{0.04} & \textbf{0.06} & \textbf{0.08} & \textbf{0.10} \\
    \midrule
    20,000  & 229.7 {\scriptsize$\pm 26.4$} & \textbf{240.9} {\scriptsize$\pm 37.1$} & 234.9 {\scriptsize$\pm 15.0$} & 217.9 {\scriptsize$\pm 38.6$} & 226.1 {\scriptsize$\pm 23.7$} \\
    40,000  & 231.4 {\scriptsize$\pm 44.4$} & 224.5 {\scriptsize$\pm 34.4$} & 231.3 {\scriptsize$\pm 30.8$} & 227.0 {\scriptsize$\pm 23.1$} & 225.5 {\scriptsize$\pm 22.5$} \\
    60,000  & 217.4 {\scriptsize$\pm 40.2$} & 231.2 {\scriptsize$\pm 29.9$} & \textbf{240.5} {\scriptsize$\pm 55.0$} & 215.7 {\scriptsize$\pm 27.9$} & \textbf{240.7} {\scriptsize$\pm 35.3$} \\
    80,000  & 233.6 {\scriptsize$\pm 42.9$} & 231.8 {\scriptsize$\pm 35.7$} & 225.2 {\scriptsize$\pm 42.6$} & 220.9 {\scriptsize$\pm 17.6$} & 231.3 {\scriptsize$\pm 54.4$} \\
    100,000 & 224.0 {\scriptsize$\pm 32.1$} & 201.8 {\scriptsize$\pm 31.9$} & 217.0 {\scriptsize$\pm 48.5$} & \textbf{241.6} {\scriptsize$\pm 38.4$} & 229.2 {\scriptsize$\pm 29.5$} \\
    \bottomrule
  \end{tabular}
\end{table}

\myadded{
\subsection{Impact of Decay Strategy}
We further investigate the influence of the weight decay schedule on performance. To this end, we compare our default \textbf{Linear Decay} against \textbf{Exponential} and \textbf{Polynomial} variants. The specific formulations are defined as follows:
\begin{itemize}
    \item \textbf{Linear (Ours):} $w(t) = \max(w_{\min}, 1 - t/T)$, providing a constant rate of importance reduction.
    \item \textbf{Exponential:} $w(t) = \max(w_{\min}, \exp(-t/\tau))$, where $\tau=1$, modeling rapid initial forgetting.
    \item \textbf{Polynomial:} $w(t) = \max(w_{\min}, (1 - t/T)^p)$, where $p=2$, penalizing older samples more aggressively than the linear approach.
\end{itemize}
The empirical results on \texttt{Humanoid-run} are summarized in Table~\ref{tab:decay_strategies}. Our proposed Linear Decay strategy significantly outperforms alternative schedules. Notably, both Exponential and Polynomial decay perform worse than the SAC baseline, suggesting that overly aggressive weight reduction disrupts the learning stability required for high-dimensional control tasks.}

\begin{table}[htbp]
  \centering
  \caption{\myadded{Performance comparison of different decay strategies on \texttt{Humanoid Run}. The relative difference is calculated with respect to our Linear Decay method.}}
  \label{tab:decay_strategies}
  \renewcommand{\arraystretch}{1.2} 
  \begin{tabular}{lcc}
    \toprule
    \textbf{Decay Strategy} & \textbf{Episode Return} & \textbf{vs. Linear SWD} \\
    \midrule
    \textbf{Linear Decay (Ours)} & \textbf{229.01} $\pm$ \textbf{37.43} & -- \\
    SAC (Baseline) & 190.46 $\pm$ 7.99\phantom{0} & $-16.8\%$ \\
    Exponential Decay ($\tau = 1$) & 187.04 $\pm$ 29.85 & $-18.3\%$ \\
    Polynomial Decay ($p = 2$) & 132.91 $\pm$ 11.12 & $-42.0\%$ \\
    \bottomrule
  \end{tabular}
\end{table}
\end{document}

%% file: math_commands.tex
\usepackage{amsmath,amsfonts,bm}

\def\eqref#1{equation~\ref{#1}}

\def\1{\bm{1}}

\DeclareMathAlphabet{\mathsfit}{\encodingdefault}{\sfdefault}{m}{sl}
\SetMathAlphabet{\mathsfit}{bold}{\encodingdefault}{\sfdefault}{bx}{n}

\def\sA{{\mathbb{A}}}

\def\sS{{\mathbb{S}}}

\def\sZ{{\mathbb{Z}}}

